\documentclass[11pt]{article}
\usepackage{psfig}
\usepackage{amsmath,amsthm}
\usepackage{geometry}                
\usepackage{mathrsfs}
\usepackage{graphicx}
\usepackage{amssymb}
\usepackage{epstopdf}
\usepackage{mathtools}
\usepackage{framed}
\usepackage{algorithm}
\usepackage{algpseudocode}
\usepackage{subcaption}
\usepackage{sidecap}
\usepackage{tikz}
\usetikzlibrary{arrows.meta}

\DeclareMathOperator*{\argmin}{\arg\!\min}
\DeclareMathOperator*{\argmax}{\arg\!\max}

\newcommand{\widesim}[2][1.5]{
  \mathrel{\overset{#2}{\scalebox{#1}[1]{$\sim$}}}
}

\usepackage{epsfig}
\usepackage{url}
\usepackage{multicol}

\usepackage{lipsum}

\newcommand\blfootnote[1]{%
  \begingroup
  \renewcommand\thefootnote{}\footnote{#1}%
  \addtocounter{footnote}{-1}%
  \endgroup
}

\usetikzlibrary{arrows.meta}

\begin{document}

\fontsize{11}{13pt}\selectfont

%
%
%
%
%
%
%

\title{On Statistical Learning of Simplices:\\Unmixing Problem Revisited}
\author{
Amir~Najafi
\footnotemark[1]
\and
Saeed~Ilchi
\footnotemark[1]
\and
Amir~H.~Saberi
\footnotemark[2]
\and
Seyed~Abolfazl~Motahari
\footnotemark[1]
\and
Babak~H.~Khalaj
\footnotemark[2]
\and
Hamid~R.~Rabiee
\footnotemark[3]
}

\date{}
\maketitle

\vspace*{-2mm}

\begin{small}
\begin{center}
$^*$Data Analytics Lab (DAL), Computer Engineering Dept.,\\
$^\ddagger$Data science and Machine learning Lab (DML), Computer Engineering Dept.,\\[0mm]
$^\dagger$Electrical Engineering Dept.,\\[1mm]
Sharif University of Technology, Tehran, Iran
\end{center}
\end{small}

\vspace*{6mm}

\newtheorem{thm}{Theorem}[section]
\newtheorem{thm2}{Theorem}
\newtheorem{corl}{Corollary}
\newtheorem{note}[thm2]{Note}
\newtheorem{lemma}{Lemma}
\newtheorem{definition}{Definition}
\newtheorem{conj}{Conjecture}

\begin{abstract}
We study the sample complexity of learning a high-dimensional simplex from a set of points uniformly sampled from its interior. Learning of simplices is a long studied problem in computer science and has applications in computational biology and remote sensing, mostly under the name of `spectral unmixing'. We theoretically show that a sufficient sample complexity for reliable learning of a $K$-dimensional simplex up to a total-variation error of $\epsilon$ is $O\left(\frac{K^2}{\epsilon}\log\frac{K}{\epsilon}\right)$, which yields a substantial improvement over existing bounds. Based on our new theoretical framework, we also propose a heuristic approach for the inference of simplices. Experimental results on synthetic and real-world datasets demonstrate a comparable performance for our method on noiseless samples, while we outperform the state-of-the-art in noisy cases.\blfootnote{An implementation of our method is available at: \texttt{github.com/seyedsaberi/simplex}}\blfootnote{E-mails:~$^*$\{najafy,silchi\}@ce.sharif.edu,~~$^\dagger$sah.saberi@ee.sharif.edu,~~\{$^*$motahari,$^\dagger$khalaj,$^\ddagger$rabiee\}@sharif.edu}
\end{abstract}

\section{Introduction}
\label{sec:intro}

High-dimensional measurements from a physical system can sometimes be thought as convex combinations of a number of unknown sources, where profiles associated with each source might not be easy to assess, separately. Even though both the source profiles and combination weights are assumed to be unknown, however, one can still attempt to infer them through collection and careful analysis of a large number of i.i.d.  measurements. This computational problem, generally known as the `unmixing problem', arises in many areas including hyper-spectral remote sensing \cite{li2015minimum,ambikapathi2011chance}, mixture modeling, and analysis of tumor heterogeneity in computational biology and bioinformatics \cite{tolliver2010robust,satas2017tumor,zuckerman2013self}. In all the above cases, unmixing the data refers to inferring all the unknown source profiles as well as the weights for each observed sample.

From a geometric point of view, this problem can be formulated as learning of a high-dimensional simplex. A $K$-dimensional simplex is defined as the set of all convex combinations of $K+1$ points in $\mathbb{R}^K$, also called vertices. For example, simplices with $K=2$ and $K=3$ correspond to triangles and tetrahedrons, respectively. Learning of a simplex refers to inferring its vertices through observing randomly chosen points from its interior. This problem is also closely related to solving the following set of equations:
\begin{equation}
\label{eq:firstEq}
\boldsymbol{\Theta}\boldsymbol{p}_i=\boldsymbol{X}_i\quad,\quad i=1,\ldots,n,
\end{equation}
where  $\boldsymbol{X}_i\in\mathbb{R}^K$ represents the $i$th observed point, each column of the $K\times\left(K+1\right)$ matrix $\boldsymbol{\Theta}$ denotes a source profile (or alternatively a vertex of the simplex), and $\boldsymbol{p}_i\in\left[0,1\right]^{K+1}$ is the $i$th weight vector with $\boldsymbol{1}^T\boldsymbol{p}_i=1$. Here, the matrix $\boldsymbol{\Theta}$ and all the weight vectors $\boldsymbol{p}_i$ are assumed to be unknown. In order to transform the problem into a statistical setting, we assume $\boldsymbol{p}_i$s are generated independently from a uniform Dirichlet distribution which means data points are distributed uniformly inside the simplex. It is easy to show that \eqref{eq:firstEq}, for any $n$, has an infinite number of solutions. However, as we show it in proceeding sections, the solution which corresponds to the minimum volume of the simplex specified by $\boldsymbol{\Theta}$ is the maximum-likelihood estimator of the above-mentioned statistical inference problem. 

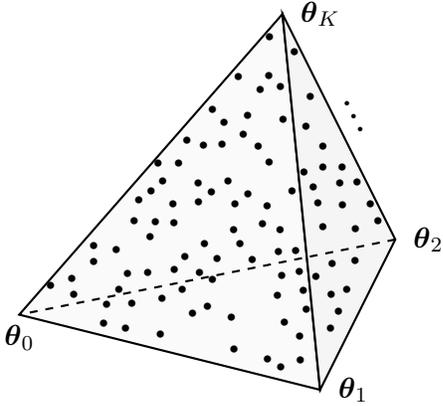
\begin{SCfigure}[3.5][t]
\begin{tikzpicture}
\draw[thick,fill=gray!4!white] (0,0) node[anchor=north]{$\boldsymbol{\theta}_0$} -- (4,-1) node[above=0mm,right=1mm]{$\boldsymbol{\theta}_1$} -- (3.5,4) node[right=1mm]{$\boldsymbol{\theta}_K$} -- cycle;
\draw[thick,fill=gray!8!white] (4,-1) -- (3.5,4) -- (5,1) node[pos=0.5,sloped,above] {$\boldsymbol{\cdots}$} node[right=1mm]{$\boldsymbol{\theta}_2$} -- cycle;
\draw[style=dashed,thick] (5,1) -- (0,0);
\foreach \Point in{
(1.90741465961273,1.73745885357984),
(2.69903859492206,2.21889977275286),
(4.03775517823762,1.12137106833491),
(4.03864635306647,2.19406934928661),
(4.24614288704699,-0.00250798953014034),
(3.20268197267953,2.94311813574237),
(3.47130335218359,1.07714703283408),
(2.86913091373367,1.43366494487506),
(4.49333893651086,1.7174227143248),
(1.71018072093036,0.512234979069354),
(4.77054205147327,1.2042611073166),
(3.17767893191886,2.19199294186685),
(4.46890423469414,1.09594322065329),
(3.32749867877415,3.64960949336287),
(3.09774285306258,0.686759558014612),
(4.14706111451999,0.220441688123797),
(2.3829937419268,1.45302231214443),
(4.26164112298117,1.70823764677055),
(2.85820097075952,-0.502100319848783),
(2.30416454584493,-0.0292698047402241),
(4.14090503693051,-0.231384473579646),
(3.29524393805099,1.40252053931201),
(3.9739245558471,1.44478265865783),
(4.2940639341424,1.91748214502483),
(1.81275111208662,1.18489118051962),
(3.8122160518018,2.46357977077624),
(3.4719807578784,2.98205372620071),
(3.49182580051464,0.477981673833099),
(2.22917323540034,2.27365829210906),
(2.6314754848722,0.69495164342519),
(3.04130940987146,1.55689143561426),
(3.65865366062066,3.44988274261963),
(1.49710541819306,0.107391325466616),
(4.66953822270851,1.43193944839711),
(2.77572347719362,1.73629797877952),
(3.03510339468836,2.60322619409367),
(3.81597215255144,0.0191078270627377),
(4.03774703682958,1.91131718605131),
(2.04710767329152,1.42801900499616),
(2.72285645243733,2.51482243394361),
(2.05660683807499,1.18195274380298),
(1.90084731347956,0.167851526774824),
(0.726422928726824,0.220751084341077),
(1.16364422348234,0.0957399228680816),
(1.29297574565518,0.81439646916941),
(2.06433352679491,0.548661057908273),
(0.990171290921141,0.877892252039494),
(2.55825203118761,1.59122931920439),
(3.32276768646299,1.97749797510596),
(2.93694671924722,0.893449692415673),
(3.23639765928263,1.16265092439203),
(4.35581129783308,0.283676194877927),
(0.701494374341553,0.441133705818892),
(3.87116155454799,2.86046968899116),
(4.16061395052986,0.637619247307414),
(2.85060083061857,1.08651148398954),
(3.47742634049089,-0.747194153914992),
(2.56855377105976,2.68097011392162),
(4.46556089105924,0.674736182097982),
(2.39606826330704,0.517471462104011),
(3.67529894399574,0.802046029587339),
(3.7281220184346,-0.335923994236394),
(2.16758136574162,0.288959558821939),
(2.56808095287322,0.415136305391028),
(1.87851481288836,-0.311823757082478),
(1.84573743723228,1.97065814457406),
(3.50872620396872,2.55537517125035),
(3.52601634790225,-0.145522882496391),
(3.30206421341264,-0.633242613245006),
(0.417119793960509,0.346542475770157),
(3.94407480200074,0.456468574416006),
(1.74126990762198,0.984247587976822),
(3.7883522893738,1.85589325290348),
(1.59012499480928,1.48235023681536),
(3.44615623084556,0.794493215281222),
(0.99124329396788,0.658526341204642),
(2.82283324980944,-0.0786503557972336),
(1.12134046589139,-0.155186161990587),
(1.75375008442578,1.60259541400249),
(4.29734682337578,1.41967227753544),
(3.90847560792292,1.65490331384737),
(2.91490511060986,3.12387590363429),
(3.43388427059138,0.209834063997793),
(2.46294050856996,0.900048206525934),
(2.97987560916042,2.34696142688778),
(3.73648017414358,-0.646228587799417),
(2.50501818999204,0.0612180307217374),
(1.53037956718171,1.20249779158606),
(1.44230135984124,0.504839352165002),
(2.20317740235785,0.730566218721534),
(3.32086475205873,3.13198952331398),
(3.62943794717161,1.60024896415198),
(2.12209919868567,1.9639127526732),
(3.7778471458275,0.271258435438607),
(2.45228178774013,2.21167259993817),
(1.41392552976034,-0.222511547810455),
(3.72295961638335,0.529771954514881),
(2.38201522699052,1.69854113238922),
(1.32424365610343,0.28472516198887)}{
    \node at \Point {\tiny\raisebox{1ex}{\textbullet}};
}
\end{tikzpicture}
\caption{\small{Learning of simplices from a geometric viewpoint: Assume a $K$-dimensional simplex specified by $K+1$ latent vertices $\boldsymbol{\theta}_0,\boldsymbol{\theta}_1,\ldots,\boldsymbol{\theta}_K\in\mathbb{R}^K$ which correspond to the columns of the vertex matrix $\boldsymbol{\Theta}$. Given $n$ random points uniformly sampled from the interior of the simplex, our aim is to find an estimator for $\boldsymbol{\Theta}$, say $\hat{\boldsymbol{\Theta}}$, such that the uniform probability measures over the simplices specified by $\boldsymbol{\Theta}$ and $\hat{\boldsymbol{\Theta}}$ have a total variation distance of at most $\epsilon$ with probability at least $1-\zeta$, for any given $\epsilon,\zeta>0$.}\label{fig:mainProblem}}
\end{SCfigure}

Assume a uniform probability distribution in $\mathbb{R}^K$ which is supported over the simplex specified by $\boldsymbol{\Theta}$. Then, our aim is to design an algorithm that with probability at least $1-\zeta$ estimates this probability distribution with a total variation error of at most $\epsilon$, for any given $\epsilon,\zeta>0$, given that the number of observations satisfies $n\ge\mathrm{poly}\left(\epsilon^{-1},\log\left(1/\zeta\right)\right)$. In this setting, $\epsilon$ also coincides with the normalized Nikodym distance, or the relative volume of the symmetric difference between the estimated simplex and the true one \cite{brunel2013adaptive}. The geometric interpretation of this problem is illustrated in Figure \ref{fig:mainProblem}. Some interesting questions in this setting would be: How many samples are required for the above-mentioned estimation problem (sample complexity)? or, How well does any particular algorithm perform if observations become noisy? In this study, we show that a sufficient sample complexity to ensure a total variation error of at most $\epsilon$ with probability at least $1-\zeta$ is 
\begin{equation}
n\ge O\left(\frac{1}{\epsilon}\left[
K^2\log\frac{K}{\epsilon}+\log\frac{1}{\zeta}
\right]\right),
\end{equation}
where achievability is guaranteed by the Maximum Likelihood Estimator (MLE). Our approach is based on tools from statistical learning theory and high-dimensional geometry. To the best of our knowledge, this result yields a significant improvement over the only existing bound of \cite{anderson2013efficient}, which is $n\ge O\left(K^{22}/\epsilon^2\right)$. However, finding the MLE requires solving a combinatoric optimization problem and is proved to have an exponential-time computational complexity \cite{packer2002np}, while the method of \cite{anderson2013efficient} is polynomial-time. Motivated by this fact, we also propose a novel continuously-relaxed version of the original ML problem, followed by a heuristic Gradient Descent (GD)-based algorithm for its numerical optimization. Theoretical sample complexity for the proposed relaxed-MLE is $O\left(\epsilon^{-2}\log\frac{1}{\epsilon}\right)$, while maintaining the same polynomial dependence on dimension $K$ as that of MLE. We have tested our method on a number of synthetic and real-world datasets. Experimental results demonstrate an acceptable performance for our method in the noiseless case, while they show a considerable superiority over rival techniques when data samples are noisy.

The paper is organized as follows: Section \ref{sec:intro:LR} reviews the related works. Section \ref{sec:notation} explains the mathematical notation and also formally defines our problem. Our proposed method and theoretical results are presented in Section \ref{sec:main}. Section \ref{sec:exp} is devoted to experimental results on both synthetic and real-world datasets. Finally, conclusions are made in Section \ref{sec:conc}.


\subsection{Related Works}
\label{sec:intro:LR}

Previous works in this domain can be divided into two separate categories: papers that follow a theoretical approach, and those trying to solve a real-world problem via heuristics.

From a theoretical viewpoint, in a seminal work on learning linear transformations \cite{frieze1996learning}, authors have proved the efficient learnability\footnote{In this paper, the term {\it {learnable}} refers to a learning task with a polynomial sample complexity w.r.t. $\epsilon$ and $\zeta$. This notation is consistent with that of \cite{mohri2012foundations}. A learnable case which also has a polynomial computational complexity is referred to as {\it {efficiently learnable}}.} of {\it {hyper-cubes}} via $O\left(K^8\right)$ samples. They suggested a possibly similar result for simplices. However, authors in \cite{anderson2013efficient} have recently proved that a sufficient sample size for efficient learnability of simplices is $n\ge O\left(K^{22}\right)$, which is the only existing bound with respect to dimension $K$ prior to this paper. 
Regarding the estimation error $\epsilon$, it is known that learning of high-dimensional polytopes requires $\tilde{O}\left(\epsilon^{-1}\right)$ \footnote{We use $\tilde{O}\left(\cdot\right)$ to hide possible logarithmic dependencies} samples \cite{brunel2016adaptive,brunel2018methods}. In particular, Theorem 1 and Corollary 1 of \cite{brunel2016adaptive} shows that the following bound holds for the expected total-variation error (or equivalently, the expected relative Nikodym error) of estimating a $K$-dimensional polytope with $r$ vertices ($r\ge K
+1$): 
\begin{equation}
\epsilon~{\leq}_{K}~\frac{r\log n}{n}.
\label{eq:revIneq}
\end{equation}
Here, ${\leq}_{K}$ indicates that the inequality holds up to a factor that only depends on dimension $K$. The bound is achievable via the ML estimator which in this case would be the minimum-volume polytope with at most $r$ vertices that includes all the data points. To the best of our knowledge, the (seemingly) optimal dependence of sample complexity on $K$ has been left as an open problem. In this paper, we introduce a new approach to solve this problem for the special case of simplices, i.e. when $r=K+1$. Our method is based on statistical analysis of the $K+1$ facets of a $K$-dimensional simplex; However, this technique may not work for the broader case of $K$-dimensional $r$-polytopes with $r>K+1$, due to the rapidly increasing number of facets as a function of vertex number $r$ in high dimensions. To the best of our knowledge, the best sample complexity results for learning of $r$-polytopes when $r$ is strictly larger than $K+1$ is still \eqref{eq:revIneq}.

A nice line of research focuses on estimating convex bodies in general (and simplices, in particular) through the {\it {convex~hull}} of $n$ uniformly sampled data points, denoted by $\hat{K}_n$. For example, it has been shown that the random polytope $\hat{K}_n$ is rate optimal over the class of all convex bodies in $\mathbb{R}^K$ in a minimax sense, when Nikodym metric is used \cite{brunel2016adaptive}. There are also other studies on the estimation power and/or the expected number of vertices of Convex Hull Estimators (CHE) for different classes of geometric shapes such as ellipsoids and polytopes; See  \cite{groemer1974mean,barany1993random,barany1988convex}. As long as one is concerned about estimating the data generating distribution as a whole, without any constraint on the shape of its support, CHE is shown to attain the following total-variation error bound in an efficient (polynomial) time (Theorem 2 of \cite{barany1993random}):
\begin{equation}
\label{eq:convEst}
\epsilon
~\lesssim~
\frac{K^2}{n}\left(\frac{\log n}{K+1}\right)^{K-1}
+
\frac{\left(\log n\right)^{K-1}}{n}.
\end{equation}
We show that, for the case of simplices, MLE achieves a better sample complexity bound by removing the exponential terms w.r.t. $K$ in the r.h.s. of \eqref{eq:convEst}. However, unlike CHE, MLE is NP-hard. Once again, we should reiterate that CHE does not necessarily output a simplex and thus is not directly linked to the approach that we follow in this paper. In this work, we are interested in estimators in the class of all $K$-dimensional {\it {simplices}} and not the far richer class of all polytopes in $\mathbb{R}^K$.

Learning of simplices also has close ties with Nonnegative Matrix Factorization (NMF) research \cite{ang2019algorithms,wu2019stochastic}. In fact, with proper normalization, one can transform many NMF problems into learning of a simplex (or more generally, a polytope) as long as there exists a minimum-volume constraint or related regularizers in the formulation of their corresponding objective functions. A nice survey on NMF methods used for data clustering can be found in \cite{li2018nonnegative}. Most of the existing works on NMF focus on its real-world applications, while fewer interesting works have considered its theoretical identifiability issues \cite{fu2018identifiability,fu2019nonnegative}. In addition, a nice study on the generalization theory of dictionary learning and other matrix factorizations can be found in \cite{gribonval2015sample}. In proceeding sections, we propose a relaxed scheme for learning of simplices that has some mathematical similarities to the regularization techniques that are known to be useful in this line of research; See, for example, the volume regularizer proposed in \cite{ang2019algorithms}. It is worth mentioning that the output of NMF, in general, may span a larger vector space than that of the data samples, which is another subtle difference between learning of simplices and NMF research. In any case, not much work has been done so far on the sample complexity of NMF problems, which means our work could be helpful for researchers working in this area.

From an algorithmic point of view, the ML estimator in the noiseless setting is equivalent to finding the minimum volume simplex that contains all the data points. This task is shown to be NP-hard, irrespective of the input representation which could be either via  facets or vertices \cite{packer2002np}. In this regard, \cite{chan2009convex} introduced an LP-relaxation that computes the determinant of a modified vertex matrix instead of the volume. Determinant can then be written as the sum of matrix co-factors and consequently optimized in a row-wise manner. However, authors do not provide any theoretical guarantees, while experiments are used for justification of their method. Authors of \cite{nguyen2009learning} studied a similar problem in lattice-based cryptography, where inputs are generated from a parallelepiped instead of a simplex. Our problem also shares some similarities to the Blind Source Separation (BSS) research \cite{gribonval2015sample}. In fact, \cite{frieze1996learning} has already shown that at least one particular problem from Independent Component Analysis (ICA) area can be transformed into learning of a high-dimensional hyper-cube from uniform samples. However, assumptions on the statistical generation of weight vectors $\boldsymbol{p}_i$ are still crucially different between a hyper-cube and a simplex. As a result, in BSS research one usually employs high-order statistical moments for capturing the source signals \cite{hyvarinen2004independent}, while we follow a different approach which is described in Section \ref{sec:main}.

From a more applied perspective, learning of simplices is of practical interest in a variety of applications, including hyper-spectral imaging \cite{zhang2017robust}. Hyper-spectral cameras capture electromagnetic energy scatters from a scene in multiple wavelengths, where the spectrum of each pixel is a combination of reflected (or radiated) patterns of its basic elements \cite{ambikapathi2011chance,agathos2014robust}. Recently, this field has attracted an intensive research interest, which are mostly centered on analysis of the minimum volume simplex, c.f. \cite{li2015minimum,ambikapathi2011chance,bioucas2012hyperspectral,lin2013endmember}. Another tightly related application to our paper is the analysis of common patterns in cancer tumors \cite{satas2017tumor}. Tumors can have a high cell-type heterogeneity, and a sufficient knowledge of the genetic characteristics that correspond to each cell-type is vital for recommending an effective treatment. However, biological datasets are mostly in {\it {bulk format}} which means each sample is an aggregation of cell populations from all the different cell-types with unknown combination weights. Again, the idea of finding the smallest inclusive simplex for capturing these hidden characteristics is exploited in several recent articles \cite{tolliver2010robust,satas2017tumor,zuckerman2013self,schwartz2010applying}.

\section{Notation and Definitions}
\label{sec:notation}

For $K\in\mathbb{N}$, let us denote the number of vertices in our model by $K+1$. Without loss of generality\footnote{Points within a $K$-simplex lie on a $K$-dimensional linear subspace, which is almost surely identifiable and efficiently learnable as long as $n\ge K+1$.}, the dimensionality of data points can be assumed to be $K$. Let us denote $\boldsymbol{\Phi}$ as the set of all discrete $\left(K+1\right)$-dimensional probability mass functions, i.e. $\boldsymbol{\Phi}\triangleq \left\{\boldsymbol{p}\in\mathbb{R}^{K+1}\big\vert~ \sum_{k}p_k=1,~p_k\ge0 \right\}$. 
$\boldsymbol{\Phi}$ is generally referred to as the {\it {standard simplex}}. 
Assume $\boldsymbol{\theta}_0,\ldots,\boldsymbol{\theta}_K\in\mathbb{R}^K$, and let
$\boldsymbol{\Theta}\triangleq\left[\boldsymbol{\theta}_0\vert\cdots\vert\boldsymbol{\theta}_K\right]$ to denote the vertex matrix, where the $k$th column of $\boldsymbol{\Theta}$ represents the $k$th vertex.
We define $\mathcal{S}$ as a $K$-simplex with the vertex matrix $\boldsymbol{\Theta}$ as
\begin{equation}
\mathcal{S}=\mathcal{S}\left(\boldsymbol{\Theta}\right)\triangleq
\left\{\boldsymbol{x}\in\mathbb{R}^K\bigg\vert~
\boldsymbol{x}=\sum_{k}p_k\boldsymbol{\theta}_{k},~\boldsymbol{p}\in\boldsymbol{\Phi}
\right\}.
\end{equation}
Also, $\mathbb{S}_K$ represents the set of all $K$-simplices in $\mathbb{R}^K$. For a $K$-simplex $\mathcal{S}\in\mathbb{S}_K$, $\mathrm{Vol}\left(\mathcal{S}\right)$ denotes the Lebesgue measure (or volume) of $\mathcal{S}$. For $\mathcal{S}\in\mathbb{S}_K$, let $\mathcal{S}_{-k}$ be the $k$th polygonal facet of $\mathcal{S}$, or equivalently, the $\left(K-1\right)$-simplex obtained by removing the $k$th vertex of $\mathcal{S}$. 

In Section \ref{sec:main}, we argue that in continuously-relaxed regimes, learning of a simplex may heavily depend on its level of geometric regularity. In other words, simplices with more or less equally-sized facets are much easier to learn, compared to those which include very acute corners. Inspired by the {\it {isoperimetric}} inequality in geometry \cite{osserman1978isoperimetric}, let us define the $\left(\underline{\lambda},\bar{\lambda}\right)$-isoperimetricity property for a $K$-simplex as follows:
\begin{definition}
\label{def:shape}
A $K$-simplex $\mathcal{S}\left(\boldsymbol{\Theta}\right)\in\mathbb{S}_K$ for $\boldsymbol{\Theta}\triangleq\left[\boldsymbol{\theta}_0\vert\cdots\vert\boldsymbol{\theta}_K\right]$ is said to be $\left(\underline{\lambda},\bar{\lambda}\right)$-isoperimetric for some $\underline{\lambda},\bar{\lambda}>0$, if
\begin{align*}
\max_{k, k'}~\left\Vert
\boldsymbol{\theta}_k-\boldsymbol{\theta}_{k'}
\right\Vert_2
~\leq~
\underline{\lambda} K \mathrm{Vol}\left(\mathcal{S}\right)^{\frac{1}{K}}
\quad,\quad
\max_{k}~\mathrm{Vol}\left(\mathcal{S}_{-k}\right)
~\leq~
\bar{\lambda}\mathrm{Vol}\left(\mathcal{S}\right)^{\frac{K-1}{K}}.
\end{align*}
\end{definition}
The essence of $\left(\underline{\lambda},\bar{\lambda}\right)$-isoperimetricity property is to ensure that a given $K$-simplex is comparably stretched along all the $K$ dimensions of $\mathbb{R}^K$. According to Lemma \ref{lemma:isoperimBound} (see Appendix \ref{app:aux}), for a perfectly regular simplex with equal side lines and all $K\in\mathbb{N}$, $\underline{\lambda}$ and $\bar{\lambda}$ can be chosen to be as small as $1$ and $e$, respectively. 

For any simplex $\mathcal{S}\in\mathbb{S}_K$, let $\mathcal{H}_k\triangleq\left\{\boldsymbol{x}\in\mathbb{R}^K\vert~\boldsymbol{w}_k^T\boldsymbol{x}+b_k=0\right\}$ for $k=0,\ldots,K$ represent the hyper-plane which encompasses the $k$th facet of $\mathcal{S}$. This way, $\boldsymbol{w}_k\in\mathbb{R}^K$ denotes a normal vector (with $\left\Vert\boldsymbol{w}_k\right\Vert_2=1$) and $b_k\in\mathbb{R}$ is a corresponding bias value. Vectors $\boldsymbol{w}_k$ are assumed to be outward w.r.t. $\mathcal{S}$. Then, the following defines the {\it {planar distance}} of a point in $\mathbb{R}^K$ from the simplex $\mathcal{S}$:
\begin{definition}
For any point $\boldsymbol{x}\in\mathbb{R}^K$, let us define the planar distance of $\boldsymbol{x}$ from $\mathcal{S}$, denoted by $d_{\mathcal{S}}\left(\boldsymbol{x}\right)$, as
\begin{equation}
d_{\mathcal{S}}\left(\boldsymbol{x}\right)
\triangleq
\max\left\{
0,~
\max_{k}~\boldsymbol{w}^T_k\boldsymbol{x}+b_k
\right\}.
\end{equation}
\label{def:planarDist}
\end{definition}
\begin{SCfigure}[1.3][t]
\begin{tikzpicture}
\draw[thick,fill=gray!20!white] (0,0) node[above=3mm,right=6mm]{$\mathcal{S}\left(\boldsymbol{\Theta}\right)$} -- (3,-1) -- (1,2) -- cycle;
\draw[style=dashed,thick] (-0.7050344,-0.2920348) -- (4.2243721,-1.9351703) -- (0.9381011,2.9942362) -- cycle;
\draw[->,thick] (1.083975,0.9726499) -- (1.50,1.25) node[above=1.5mm,right=-0.15mm]{$\delta$};
\draw[->,thick] (2.332050,1.80470) --  (1.916025,1.52735);
\end{tikzpicture}
\hspace*{7mm}
\caption{\small{Graphical illustration of the {\it {planar distance}} $d_{\mathcal{S}\left(\boldsymbol{\Theta}\right)}\left(\cdot\right)$ in Definition \ref{def:planarDist}. The gray triangle shows an arbitrary two-dimensional simplex $\mathcal{S}\left(\boldsymbol{\Theta}\right)\in\mathbb{S}_2$. For some $\delta>0$, the dashed contour corresponds to the set  $\left\{\boldsymbol{x}\in\mathbb{R}^2\vert~d_{\mathcal{S}\left(\boldsymbol{\Theta}\right)}\left(\boldsymbol{x}\right)=\delta\right\}$. As can be seen, the set of points in $\mathbb{R}^K$ with the same positive planar distance from a simplex $\mathcal{S}\in\mathbb{S}_K$ is in fact the boundary of another simplex whose facets are parallel to those of $\mathcal{S}$.}\label{fig:planarDist}}
\end{SCfigure}
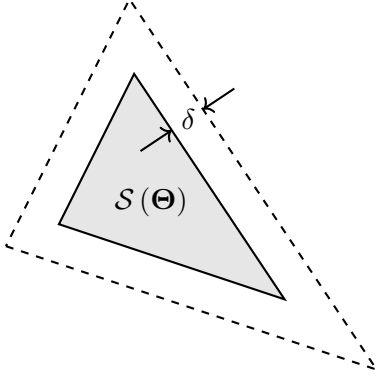
We always have $d_{\mathcal{S}}\left(\boldsymbol{x}\right)\ge0$, and since each $\boldsymbol{w}_k$ is an outward normal vector w.r.t. $\mathcal{S}$, $d_{\mathcal{S}}\left(\boldsymbol{x}\right)=0$ holds if and only if $\boldsymbol{x}\in\mathcal{S}$. Also, it is easy to check that for any fixed $d_0\ge0$, the set $\left\{\boldsymbol{x}\in\mathbb{R}^K\vert~d_{\mathcal{S}}\left(\boldsymbol{x}\right)\leq d_0\right\}$ is also a simplex which encompasses $\mathcal{S}$ (see Figure \ref{fig:planarDist}).
For a simplex $\mathcal{S}\in\mathbb{S}_K$, $\mathbb{P}_{\mathcal{S}}$ denotes the uniform probability distribution which is supported over $\mathcal{S}$. Let $\rho_{\mathcal{S}}$ be the probability density function associated to $\mathbb{P}_{\mathcal{S}}$, then
\begin{equation}
\rho_{\mathcal{S}}\left(\boldsymbol{x}\right)
\triangleq
\frac{\boldsymbol{1}_{\mathcal{S}}\left(\boldsymbol{x}\right)}
{\mathrm{Vol}\left(\mathcal{S}\right)}\quad\mathrm{for}\quad
\forall\boldsymbol{x}\in\mathbb{R}^K,
\end{equation}
where $\boldsymbol{1}_{\mathcal{S}}\left(\boldsymbol{x}\right)$ is the indicator function of the $K$-simplex $\mathcal{S}$ which returns $1$ if $\boldsymbol{x}\in\mathcal{S}$, and zero otherwise.
Also, $\mathcal{D}_{\mathrm{TV}}\left(\cdot,\cdot\right)$ represents the total-variation distance between two probability distributions. 

Now, we can formally define the problem that has been tackled in this paper.
Assume $\boldsymbol{X}_1,\ldots,\boldsymbol{X}_n\in\mathbb{R}^K$ to be $n$ i.i.d. samples which are generated uniformly from $\mathcal{S}_T\in\mathbb{S}_K$, i.e.
$\boldsymbol{X}_1,\ldots,\boldsymbol{X}_n\widesim[1.5]{}
\mathbb{P}_{\mathcal{S}_T}$. The problem is to find an approximation of $\mathcal{S}_T$, denoted by ${\mathcal{S}}^*$, from the dataset $\boldsymbol{D}\triangleq\left\{\boldsymbol{X}_1,\ldots,\boldsymbol{X}_n\right\}$ such that with probability at least $1-\zeta$ the total variation between 
$\mathbb{P}_{\mathcal{S}^*}$ and $\mathbb{P}_{\mathcal{S}_T}$
is less than $\epsilon$. 

\section{Statistical Learning of Simplices: Method and Results}
\label{sec:main}


In this section, we first formulate the MLE for estimation of $\mathcal{S}_T$ from the data data points in $\boldsymbol{D}$. We then give a new sample complexity bound for MLE which seems to be optimal with respect to both $K$ and $\epsilon$. Then, we propose a novel relaxation for the ML estimator that has the same sample complexity bound as that of MLE, at least with respect to dimension $K$.

Given a dataset $\boldsymbol{D}$, the ML estimator for $\mathcal{S}_T$, which we denote by $\mathcal{S}^*_{\mathrm{ML}}=\mathcal{S}^*_{\mathrm{ML}}\left(\boldsymbol{D}\right)$ 
would be the maximizer of the log-likelihood function $\log\mathbb{P}_{\mathcal{S}}\left(\boldsymbol{D}\right)$ according to the following formulation:
\begin{gather}
\label{eq:mainMLmin}
\mathcal{S}^{*}_{\mathrm{ML}}
\triangleq
\argmax_{\mathcal{S}\in\mathbb{S}_K}~\left\{
\log\rho_{\mathcal{S}}\left(\boldsymbol{D}\right)=
\log\prod_{i=1}^{n}
\rho_{\mathcal{S}}\left(\boldsymbol{X}_i\right)
=\sum_{i=1}^{n}\log\boldsymbol{1}_{\mathcal{S}}\left(\boldsymbol{X}_i\right)
-n\log\mathrm{Vol}\left(\mathcal{S}\right)\right\}.
\end{gather}
It is easy to see that the optimization in \eqref{eq:mainMLmin} aims to find the smallest simplex (in terms of volume), which contains all the data points in $\boldsymbol{D}$: For a simplex $\mathcal{S}\in\mathbb{S}_K$, even if a single data point happens to be outside of $\mathcal{S}$, then the log-likelihood of $\mathcal{S}$ being the true simplex becomes $-\infty$. Moreover, among those simplices that contain all the data points, the one(s) with the smallest volume would be the maximizer(s) of \eqref{eq:mainMLmin}. Regardless of the computational burden that is required to find the MLE, one can still be interested in a sufficient number of samples in $\boldsymbol{D}$ which guarantees the suitability of MLE in a PAC sense. In this regard, Theorem \ref{corl:ML} states one of our main contributions in this paper, where we derive a new sample complexity bound for MLE as a function of both dimension $K$ and estimation error $\epsilon$.
\begin{thm2}[Sample Complexity of MLE]
\label{corl:ML}
Assume a $K$-simplex $\mathcal{S}_T\in\mathbb{S}_K$ and let $\boldsymbol{X}_1,\ldots,\boldsymbol{X}_n$ be $n$ i.i.d. samples drawn from $\mathbb{P}_{\mathcal{S}_T}$. Assume there exist $\epsilon,\zeta>0$, such that
\begin{equation*}
n\ge O\left( \frac{1}{\epsilon}\left[{K^2\log\left(\frac{K}{\epsilon}\right)+\log\frac{1}{\zeta}}\right]\right).
\end{equation*} 
Then, with probability at least $1-\zeta$, the maximum likelihood estimator of $\mathcal{S}_T$, denoted by $\mathcal{S}^*_{\mathrm{ML}}$, satisfies
$\mathcal{D}_{\mathrm{TV}}\left(\mathbb{P}_{\mathcal{S}_T},\mathbb{P}_{\mathcal{S}^*_{\mathrm{ML}}}\right)\leq \epsilon$.   
\end{thm2}

The proof of Theorem \ref{corl:ML} is given in Appendix \ref{appendix:Theorems}. Sample complexity is shown to be $O\left(K^2\log K\right)$ w.r.t. the dimension $K$, while the dependence on estimation error $\epsilon$ has remained optimal \cite{brunel2018methods}. Interestingly, the given guarantees on the accuracy of MLE hold regardless of the shape of the simplex and does not impose any geometric constraints on $\mathcal{S}_T$. 

However, solving for \eqref{eq:mainMLmin} is proved to be NP-hard, and thus impractical in real-world situations \cite{packer2002np}. In the remainder of this section, we try to propose a heuristic framework for the estimation of simplices, which is practically amenable and at the same time enjoys from a number of weaker theoretical guarantees. A solution for the computational hardness of ML estimator is to replace the objective function in \eqref{eq:mainMLmin} with a continuously-relaxed surrogate. In order to do so, first it should be noted that the maximization in \eqref{eq:mainMLmin} is equivalent to the following constrained minimization problem:
\begin{align}
\mathcal{S}^*_{\mathrm{ML}}~=~
\argmin_{\mathcal{S}\in\mathbb{S}_K}&\quad\mathrm{Vol}\left(\mathcal{S}\right)
\nonumber \\
\mathrm{subject~to}&\quad
d_{\mathcal{S}}\left(\boldsymbol{X}_i\right)=0~,~\forall i.
\label{eq:MLnoiseless}
\end{align}
The above formulation directly follows from the explanations that we gave earlier on how ML algorithm attempts to find the optimal simplex. It should be noted that any other distance measure rather than $d_{\mathcal{S}}\left(\cdot\right)$ can also be used in \eqref{eq:MLnoiseless}; However, this particular choice significantly helps us to derive our sample complexity bounds in the proceeding parts of the paper.
Motivated by the formulation  in \eqref{eq:MLnoiseless} and also the idea of Lagrangian relaxation, let us propose the Continuously-Relaxed Risk (CRR) in order to approximate $\mathcal{S}_T$ as follows:

\begin{definition}[Continuously-Relaxed Risk]
\label{def:proposed}
Assume a dataset $\boldsymbol{D}=\left\{\boldsymbol{X}_1,\ldots,\boldsymbol{X}_n\right\}$ in $\mathbb{R}^K$, parameter $\gamma\ge0$, and an increasing and integrable function $\ell:\mathbb{R}\rightarrow\mathbb{R}$. Then, the empirical Continuously-Relaxed Risk $\hat{R}_{\mathrm{CRR}}:\mathbb{S}_K\rightarrow\mathbb{R}$ is defined as:
\begin{equation}
\hat{R}_{\mathrm{CRR}}\left(\mathcal{S};\boldsymbol{D},\gamma,\ell\right)\triangleq
\frac{1}{\sqrt{n}}\sum_{i=1}^{n}\ell\left(
d_{\mathcal{S}}\left(\boldsymbol{X}_i\right)
\right) +
\gamma\mathrm{Vol}\left(\mathcal{S}\right).
\label{eq:RhatCRR}
\end{equation}
Also, let us define
\begin{equation}
\mathcal{S}^*=\mathcal{S}^*\left(\boldsymbol{D},\gamma,\ell\right)~\triangleq~
\argmin_{\mathcal{S}\in\mathbb{S}_K}~
\hat{R}_{\mathrm{CRR}}\left(\mathcal{S};\boldsymbol{D},\gamma,\ell\right),
\label{eq:proposed}
\end{equation}
as the Continuously-Relaxed Estimator (CRE) of $\mathcal{S}_T$.
\end{definition}
Parameter $\gamma>0$ and function $\ell\left(\cdot\right)$ can be freely chosen, except that $\ell\left(\cdot\right)$ must be increasing and integrable; However, $\ell$ does not need to be strictly increasing or even continuous in general. In any case, choosing a differentiable $\ell$ makes $\hat{R}_{\mathrm{CRR}}\left(\mathcal{S}\left(\boldsymbol{\Theta}\right);\boldsymbol{D},\gamma,\ell\right)$ to be differentiable w.r.t. the vertex matrix $\boldsymbol{\Theta}$ almost everywhere\footnote{As shown in Section \ref{sec:numOpt}, the subset of simplices in $\mathbb{S}_K$ for which the gradient of $\hat{R}_{\mathrm{CRR}}$ does not exist has a zero Lebesgue measure.}. In Section \ref{sec:numOpt}, we derive an efficient technique to numerically compute the derivatives of $\hat{R}_{\mathrm{CRR}}$ w.r.t. the vertex matrix $\boldsymbol{\Theta}$ which works for all differentiable functions $\ell\left(\cdot\right)$. We then propose a heuristic GD-based scheme to find a local minimizer of \eqref{eq:proposed}.
The following theorem gives us a sample complexity bound for learning of $S_T$ via \eqref{eq:proposed}. Our results hold for any increasing and integrable function $\ell:\mathbb{R}\rightarrow\mathbb{R}$.

\begin{thm2}[Sample Complexity for General $\ell$]
Assume a $K$-simplex $\mathcal{S}_T$ with Lebesgue measure $V_T\triangleq\mathrm{Vol}\left(\mathcal{S}_T\right)$, which is $\left(\underline{\lambda},\bar{\lambda}\right)$-isoperimetric for some $\underline{\lambda},\bar{\lambda}\ge0$. Also, assume $\boldsymbol{X}_1,\boldsymbol{X}_2,\ldots,\boldsymbol{X}_n$ to be $n$ i.i.d. samples drawn from $\mathbb{P}_{\mathcal{S}_T}$. Assume for $\zeta,\epsilon>0$, the following condition holds for $n$:
\begin{equation*}
n\ge
\left(
\frac{
6\ell\left(3\underline{\lambda}KV^{\frac{1}{K}}_T\right)
\left(
\sqrt{K^2\log\frac{ne}{K}}
+
\sqrt{\log\frac{1}{\zeta}}
\right) + \gamma V_T\epsilon}
{\epsilon L\left(\frac{\epsilon V^{1/K}_T}{\left(K+1\right)\bar{\lambda}}\right)}
\right)^2,
\end{equation*}
where $L\left(x\right)\triangleq\frac{1}{x}\int_{0}^{x}\ell\left(u\right)\mathrm{d}u - \ell\left(0\right)$. Then, with probability at least $1-\zeta$ we have $\mathcal{D}_{\mathrm{TV}}\left(\mathbb{P}_{\mathcal{S}_T},\mathbb{P}_{\mathcal{S}^*}\right)\leq\epsilon$, where $\mathcal{S}^*$ is an optimizer of \eqref{eq:proposed}.
\label{thm:PACnoiseless}
\end{thm2}

Proof of Theorem \ref{thm:PACnoiseless} is given in Appendix \ref{appendix:Theorems}. Similar to Theorem \ref{corl:ML}, the proof includes mathematical techniques from high-dimensional geometry and calculus, as well as tools from Vapnik-Chervonenkis (VC) theory of statistical learning \cite{mohri2012foundations}. The general statement of Theorem \ref{thm:PACnoiseless}, for particular choices of the smoothing function $\ell$ and/or under certain asymptotic regimes can be extended to some interesting side results. For example, with appropriate choices of $\gamma$, one can always bound the total-variation error below $O\left(\sqrt{\log n/n}\right)$. Another interesting issue is the mathematical dependence of sample complexity on the geometric regularity of $\mathcal{S}_T$. In fact, searching for functions $\ell$ that decrease (or even obviate) the need to put geometric restrictions on the true simplex could be a nice research direction which goes beyond the scope of this paper. The following corollary is perhaps the most notable consequence of Theorem \ref{thm:PACnoiseless}.
\begin{corl}[Sample Complexity of Soft-ML]
\label{corl:softML}
Assume a $K$-simplex $\mathcal{S}_T\in\mathbb{S}_K$ and let $\boldsymbol{X}_1,\ldots,\boldsymbol{X}_n$ to be $n$ i.i.d. samples drawn from $\mathbb{P}_{\mathcal{S}_T}$. Also, assume $\mathcal{S}_T$ is $\left(\underline{\lambda},\bar{\lambda}\right)$-isoperimetric for some bounded $\underline{\lambda},\bar{\lambda}>0$. For $\epsilon,\zeta>0$ and parameter $\gamma>0$, let function $\ell:\mathbb{R}\rightarrow\mathbb{R}$ be
\begin{equation*}
\ell\left(u\right)\triangleq 1-e^{-bu},\quad \forall u\in\mathbb{R},
\end{equation*}
with $b\triangleq \frac{K}{\epsilon}$, and also assume
\begin{equation*}
n\ge_{\gamma,\Bar{\lambda},\underline{\lambda}}
O\left(
\frac{1}{\epsilon^2}
\left[
K^2\log\left(\frac{K}{\epsilon}\right)+\log\frac{1}{\zeta}
\right]\right),
\end{equation*}
where $\ge_{\gamma,\Bar{\lambda},\underline{\lambda}}$ means the inequality holds up to a factor that only depends on the mentioned parameters. Then, with probability at least $1-\zeta$ the minimizer of \eqref{eq:proposed}, denoted by $\mathcal{S}^*$, satisfies the inequality 
$\mathcal{D}_{\mathrm{TV}}\left(\mathbb{P}_{\mathcal{S}_T},\mathbb{P}_{\mathcal{S}^*}\right)\leq \epsilon$.
\end{corl}
Proof of Corollary \ref{corl:softML} is given in Appendix \ref{app:aux}. In Corollary \ref{corl:softML}, we show that for at least one class of differentiable functions $\ell$, the sample complexity for acquiring a reliable solution to \eqref{eq:proposed} has the same order with that of MLE in terms of parameter $\zeta$ and dimension $K$. However, the optimal dependence of sample complexity w.r.t. $\epsilon$, which is $\Tilde{O}\left(\epsilon^{-1}\right)$, has been deteriorated to $\Tilde{O}\left(\epsilon^{-2}\right)$. We conjecture that this sub-optimal dependence cannot be improved, since prior works on learning of other geometric shapes, such as hyper-cubes in \cite{frieze1996learning}, have also concluded that by relaxing the combinatoric and exponential-time algorithms the dependence on $\epsilon$ would be sacrificed.

By using the Soft-ML setting of Corollary \ref{corl:softML}, the optimization problem of Definition \ref{def:proposed} becomes differentiable w.r.t. $\boldsymbol{\Theta}$ almost everywhere and thus one can apply any heuristic continuous optimization tool, such as Gradient Descent (GD), to solve it. It should be noted that this result is not in contrast with the computational hardness of ML for two reasons: first, Corollary \ref{corl:softML} does not say anything about the convexity of \eqref{eq:RhatCRR}. Second, even if \eqref{eq:RhatCRR} happens to be convex with the above choice of $\ell\left(\cdot\right)$, we are still tackling a different problem compared to finding the {\it {smallest inclusive simplex}} which is what ML estimation attempts to do.

A thorough convexity analysis of our method goes beyond the scope of this paper. As a result, we make no claims regarding the computational complexity of our continuous-relaxation in Definition \ref{def:proposed}. Instead, we derive an efficient scheme to compute the gradient of \eqref{eq:proposed} w.r.t. almost any vertex matrix $\boldsymbol{\Theta}$, and then show that a simple heuristic approach, such as GD, can achieve a very good performance in practice.
%

\subsection{Numerical Optimization}
\label{sec:numOpt}

We show how to numerically compute the gradient of \eqref{eq:RhatCRR} in order to employ it in the GD-based optimization of Algorithm \ref{alg:rcchgd}. Later in Section \ref{sec:exp}, we have empirically investigated the convergence and effectiveness of Algorithm \ref{alg:rcchgd} through several computer simulations.

According to \eqref{eq:RhatCRR}, $\hat{R}_{\mathrm{CRR}}\left(\mathcal{S}\right)$ is naturally broken into two separate terms: one controls the distance of data points from $\mathcal{S}$, and the second one is a regularization term on the volume of $\mathcal{S}$. Gradient of $\gamma\mathrm{Vol}\left(\mathcal{S}\right)$ is easy to compute, since volume can be rewritten as
\begin{equation}
\mathrm{Vol}\left(\mathcal{S\left(\boldsymbol{\Theta}\right)}\right) = \frac{1}{ K !}
\left\vert
\mathrm{det}
\bigg[\boldsymbol{\theta}_1-\boldsymbol{\theta}_0\bigg\vert\cdots\bigg\vert\boldsymbol{\theta}_K-\boldsymbol{\theta}_0\bigg]
\right\vert.
\label{eq:deriPart1}
\end{equation}
On the other hand, for any matrix $\boldsymbol{A} \in \mathrm{R}^{K \times K}$ we have $\nabla_{\boldsymbol{A}} \mathrm{det}\left(\boldsymbol{A}\right) = \mathrm{adj}^T\left(\boldsymbol{A}\right)$, where $\mathrm{adj}\left(\cdot\right)$ denotes the {\it {adjugate}} function. To ensure that the gradient is well-defined, assume $\boldsymbol{\Theta}$ to be non-degenerate, i.e. $\mathcal{S}\left(\boldsymbol{\Theta}\right)$ has a positive volume. This assumption only excludes a zero-measure (according to the Lebesgue measure) subset of solutions from $\mathbb{S}_K$. For the remaining part of \eqref{eq:RhatCRR}, one needs to compute the gradient vectors $\boldsymbol{g}_{ik}= \nabla_{\boldsymbol{\theta}_k}\ell\left(
d_{\mathcal{S}\left(\boldsymbol{\Theta}\right)}\left(\boldsymbol{X}_i\right)
\right),~i\in\left[n\right],~k\in\left[K\right]$, or alternatively, the gradient matrices $\boldsymbol{G}_i\triangleq\left[\boldsymbol{g}_{i0}\vert\cdots\vert\boldsymbol{g}_{iK}\right]~,~i\in\left[n\right]$. The following theorem gives a straightforward procedure for numerical computation of $\boldsymbol{G}_i$s.
\begin{algorithm}[t]
\caption{Learning of simplices via Gradient Descent}
\label{alg:rcchgd}
\begin{algorithmic}[1]
\Procedure {Simplex Inference}{$\boldsymbol{D}=\left\{\boldsymbol{X}_1,\ldots,\boldsymbol{X}_n\right\}, K, \ell\left(\cdot\right),\gamma, T,\alpha$}
	\State Select $\left\{i_0,i_1,\ldots,i_{K}\right\}\subset\left[n\right]$ uniformly at random.
	\State Initialize $\boldsymbol{\Theta}^{\left(0\right)} = \left[\boldsymbol{X}_{i_0}|\ldots|\boldsymbol{X}_{i_{K}}\right]$
	\For {$t=0:\cdots:T-1$}
		\State $\boldsymbol{\Theta}^{\left(t+1\right)} \leftarrow \boldsymbol{\Theta}^{\left(t\right)} - \alpha \nabla_{\boldsymbol{\Theta}}\left[
		\hat{R}_{\mathrm{CRR}}\left(\mathcal{S}\left(\boldsymbol{\Theta}\right);
		\boldsymbol{D},\ell,\gamma\right)\right]$ \Comment{According to \eqref{eq:totalGrad} and Theorem \ref{lemma:Grad}}
	\EndFor
\EndProcedure
\end{algorithmic}
\end{algorithm}

\begin{thm2}[Gradient of the Planar Distance]
\label{lemma:Grad}
Assume $\boldsymbol{X}\in\mathbb{R}^{K}$, $\mathcal{S}\left(\boldsymbol{\Theta}\right)\in\mathbb{S}_K$ with $\mathrm{Vol}\left(\mathcal{S}\left(\boldsymbol{\Theta}\right)\right)>0$, and an increasing function $\ell:\mathbb{R}\rightarrow\mathbb{R}$ which is differentiable in $\mathbb{R}^+$.
\\[2mm]
i) If $d_{\mathcal{S}\left(\boldsymbol{\Theta}\right)}\left(\boldsymbol{X}\right)=0$, assume $\boldsymbol{w}^T_k\boldsymbol{X}+b_k<0$ for all $k=0,\ldots,K$, where $\left(\boldsymbol{w}_k,b_k\right)$ represent the parameters of hyper-plane $\mathcal{H}_k$ which encompasses the $k$th facet of $\mathcal{S}\left(\boldsymbol{\Theta}\right)$. Then, the gradient w.r.t. $\boldsymbol{\Theta}$ is zero.
\\[2mm]
ii) If $d_{\mathcal{S}\left(\boldsymbol{\Theta}\right)}\left(\boldsymbol{X}\right)>0$, assume $k^*\triangleq\argmax_{k}~\boldsymbol{w}^T_k\boldsymbol{X}+b_k$ is unique. Next, one needs to compute the normal outward vector $\boldsymbol{w}_{k^*}$. Let the $K\times K$ matrix  $\boldsymbol{\Theta}_{-k^*}$ represent $\boldsymbol{\Theta}$ with its $k^*$th column removed. Then, $\boldsymbol{w}_{k^*}$ is the eigenvector corresponding to the unique zero eigenvalue of  $\left(\boldsymbol{I}-\frac{1}{K}\boldsymbol{1}\boldsymbol{1}^T\right)\boldsymbol{\Theta}^T_{-k^*}$. Without loss of generality, assume $\left\Vert\boldsymbol{w}_{k^*}\right\Vert_2=1$ and $\boldsymbol{w}_{k^*}^T\left(\boldsymbol{\theta}_{k^*} - \frac{1}{K}\boldsymbol{\Theta}_{-k^*}\boldsymbol{1}\right)\leq0$ which means $\boldsymbol{w}_{k^*}$ is outward w.r.t. $\mathcal{S}\left(\boldsymbol{\Theta}\right)$. Let
\begin{equation}
\boldsymbol{p}^*\triangleq\boldsymbol{\Theta}^{\dagger}_{-k^*}\left(\boldsymbol{X}-\boldsymbol{w}_{k^*}\boldsymbol{w}_{k^*}^T\left(\boldsymbol{X}-\frac{1}{K}\boldsymbol{\Theta}_{-k^*}\boldsymbol{1}\right)\right)\in\mathbb{R}^K,
\end{equation}
where $\dagger$ denotes the pseudo-inverse operator. Then, we have:
\begin{gather}
\nabla_{\boldsymbol{\Theta}}
\ell\left( 
d_{\mathcal{S}\left(\boldsymbol{\Theta}\right)}\left(\boldsymbol{X}\right)
\right)
= 
-\ell'\left(
d_{\mathcal{S}\left(\boldsymbol{\Theta}\right)}\left(\boldsymbol{X}\right)
\right)
\boldsymbol{w}_{k^*}\boldsymbol{p}^{*T}\boldsymbol{J}_{k^*}
~\in~\mathbb{R}^{K\times\left(K+1\right)},
\label{eq:gradLoss}
\end{gather}
where $\ell'\left(\cdot\right)$ denotes the derivative of $\ell\left(\cdot\right)$, and $\boldsymbol{J}_{k^*}$ is a $K\times\left(K+1\right)$ zero-padded identity matrix as follows: 
\begin{equation}
\boldsymbol{J}_{k^*}\triangleq\left[
\boldsymbol{e}_1\vert\cdots\vert
\boldsymbol{e}_{k^*-1}\vert\boldsymbol{0}\vert\boldsymbol{e}_{k^*}\vert
\cdots\vert\boldsymbol{e}_K
\right].
\end{equation}
Here, $\boldsymbol{e}_k\in\mathbb{R}^K$ refers to a binary vector where the $k$th component is $1$ and the rest are zero.
\end{thm2}
Proof of Theorem \ref{lemma:Grad} is given in Appendix \ref{appendix:Theorems}. It should be noted that conditions i) and ii) of Theorem \ref{lemma:Grad} hold for all $\mathcal{S}\in\mathbb{S}_K$, except for a subset with a zero Lebesgue measure. Total gradient of the continuously-relaxed risk $\hat{R}_{\mathrm{CRR}}\left(\mathcal{S}\left(\boldsymbol{\Theta}\right);\boldsymbol{D},\ell,\gamma\right)$ w.r.t. vertex matrix $\boldsymbol{\Theta}$ can then be computed as
\begin{equation}
\nabla_{\boldsymbol{\Theta}}\hat{R}_{\mathrm{CRR}}\left(\mathcal{S}\left(\boldsymbol{\Theta}\right)\right)
= 
\frac{1}{\sqrt{n}}\sum_{i=1}^{n}\boldsymbol{G}_i+
\frac{\gamma s}{K!}
\left[
\boldsymbol{0}\bigg\vert
\mathrm{adj}^T\left(\boldsymbol{\Theta}_{1:K}-\boldsymbol{\theta}_0\boldsymbol{1}^T\right)
\right]
\left(\boldsymbol{I}-\frac{\boldsymbol{1}\boldsymbol{1}^T}{K+1}\right),
\label{eq:totalGrad}
\end{equation}
where $\boldsymbol{0}$ is an all-zero vector, $\boldsymbol{1}$ is an all-one vector and $\boldsymbol{I}$ denotes the identity matrix, all with corresponding appropriate sizes. Variable $s\in\left\{\pm 1\right\}$ denotes the sign of $\mathrm{det}\left[\boldsymbol{\theta}_1-\boldsymbol{\theta}_0\vert\cdots\vert\boldsymbol{\theta}_K-\boldsymbol{\theta}_0\right]$, and $\boldsymbol{\Theta}_{1:K}$ represents the $K\times K$ matrix $\left[\boldsymbol{\theta}_1\vert\cdots\vert\boldsymbol{\theta}_K\right]$. Also, we have $\boldsymbol{G}_i\triangleq \nabla_{\boldsymbol{\Theta}}\ell\left(d_{\mathcal{S}\left(\boldsymbol{\Theta}\right)}\left(\boldsymbol{X}_i\right)\right)$. In order to avoid the zero-measure set of ill-conditioned simplices in $\mathbb{S}_K$, as described by conditions i and ii of Theorem \ref{lemma:Grad}, one can add an infinitesimally small noise to $\boldsymbol{\Theta}$ before computing the gradient matrix at each iteration. For example, assuming the numerical computations are performed by a digital processor with $B$ precision bits, the addition of a Gaussian noise drawn from $\mathcal{N}\left(\boldsymbol{0},2^{-B}\boldsymbol{I}\right)$ implies that the probability of facing an ill-conditioned simplex in Algorithm \ref{alg:rcchgd} is zero.
\section{Experimental Results}
\label{sec:exp}

This section is devoted to  testing the performance of Algorithm \ref{alg:rcchgd} on a number of  synthetic and real-world datasets. In particular, we report the results of applying our method on three different tasks: 1) analysis of convergence, noise robustness and error performance on synthetic data, 2) cell-type extraction in computational biology, and 3) spectral unmixing of remote sensing data. We also compare our method with a number of popular rival frameworks. Throughout this section, {\it {estimation error}} is measured in terms of the average Euclidean distance between the vertices of the true simplex $\mathcal{S}_T$ and those of an estimated simplex  $\hat{\boldsymbol{S}}$ \footnote{We use the notation $\hat{\mathcal{S}}$ for the output of Algorithm \ref{alg:rcchgd} instead of $\mathcal{S}^*$, since $\mathcal{S}^*$ is already defined as the global optimizer of \eqref{eq:proposed}. However, Algorithm \ref{alg:rcchgd} may converge to a local minimizer which could be different from $\mathcal{S}^*$.}. For simplicity, assume $\boldsymbol{\Theta}=\left[\boldsymbol{\theta}_0\vert\cdots\vert\boldsymbol{\theta}_K\right]$ and $\hat{\boldsymbol{\Theta}}=\left[\hat{\boldsymbol{\theta}}_0\vert\cdots\vert\hat{\boldsymbol{\theta}}_K\right]$ to be the vertex matrices of $\mathcal{S}_T$ and $\hat{\mathcal{S}}$, respectively. Then, $error$ can be mathematically formulated as follows:
\begin{equation}
error^2\triangleq 
\min_{\left(i_0,\ldots,i_K\right)}
\frac{1}{K\left(K+1\right)}\sum_{k=0}^{K}\left\Vert
\boldsymbol{\theta}_k - \hat{\boldsymbol{\theta}}_{i_k}
\right\Vert_2^2,
\label{eq:errorMeasure}
\end{equation}
where minimization is taken over all permutations $\left(i_0,\ldots,i_K\right)$ of the numbers $\left\{0,1,\ldots,K\right\}$. The main reason for using $error$ instead of the total-variation distance $\mathcal{D}_{\mathrm{TV}}\left(\cdot,\cdot\right)$ is due to the numerical hardness of computing TV distance, specially in higher dimensions. Motivated by Corollary \ref{corl:softML}, we have chosen $\ell\left(u\right)\triangleq 1-e^{-bu}$ as our smoothing loss function in Definition \ref{def:proposed}, where $b\triangleq\mathrm{diam}^{-1}\left(\boldsymbol{D}\right)$ is chosen as the inverse diameter of the input database $\boldsymbol{D}=\left\{\boldsymbol{X}_1,\ldots,\boldsymbol{X}_n\right\}$. Here, by $\mathrm{diam}\left(\boldsymbol{D}\right)$ we simply mean $\max_{i,j}\left\Vert\boldsymbol{X}_i-\boldsymbol{X}_j\right\Vert_2$. Parameters $\gamma$ and the learning rate $\alpha$ in Algorithm \ref{alg:rcchgd} are manually adjusted to optimize the performance and convergence rate. Noticeably,  we did not observe a considerable sensitivity to any of these parameters.


\subsection{Initialization and Acceleration}

We have empirically observed that by setting an appropriate initial simplex $\mathcal{S}_{\mathrm{init}}\in\mathbb{S}_K$ for Algorithm \ref{alg:rcchgd}, one can significantly improve its convergence rate.
Moreover, numerical evaluation of the gradient in \eqref{eq:totalGrad} is $O\left(nK^3\right)$ which could be computationally intensive in some applications. In this part, we propose a simple heuristic approach to circumvent the above-mentioned problems. Let $\mathrm{conv}\left(\boldsymbol{D}\right)\subseteq\boldsymbol{D}$ represent the vertices of the convex hull of data points $\boldsymbol{X}_1,\ldots,\boldsymbol{X}_n\in\boldsymbol{D}$. In other words, $\mathrm{conv}\left(\boldsymbol{D}\right)$ is the smallest subset of $\boldsymbol{D}$ with the following property:
\begin{equation}
\boldsymbol{X}_1,\ldots,\boldsymbol{X}_n\in
\left\{\sum_{i}q_i\boldsymbol{Z}_i\bigg\vert~
\forall\boldsymbol{Z}_i\in
\mathrm{conv}\left(\boldsymbol{D}\right),~
\sum_i q_i=1,~\forall q_i\ge0
\right\}.
\end{equation}
This way, one can choose $\mathcal{S}_{\mathrm{init}}\in\mathbb{S}_K$ such that its vertices are randomly selected from $\mathrm{conv}\left(\boldsymbol{D}\right)$. We have empirically observed that such initialization leads to a substantially faster convergence. Moreover, one can accelerate the execution of Algorithm \ref{alg:rcchgd} by approximating the gradient in \eqref{eq:totalGrad} via computing it only for data points in $\mathrm{conv}\left(\boldsymbol{D}\right)$ instead of the whole dataset $\boldsymbol{D}$. 

The effect of the proposed approximation for the gradient remains negligible on the performance of Algorithm \ref{alg:rcchgd} in low noise regimes. In other words, when $\gamma$ is chosen to be relatively low, which means the objective function of Algorithm \ref{alg:rcchgd} becomes more similar to that of ML estimator, the proposed algorithm tries to find a simplex that encompasses almost all of the data points. This is, however, equivalent to encompassing the convex hull of data points which suggests that the proposed approximation of gradient becomes significantly close to \eqref{eq:totalGrad} after a number of initial iterations. We already know that the convex hull of a set of $n$ points can be determined in $g_{n,K}\leq\tilde{O}\left(n^{\left\lfloor K/2\right\rfloor}\right)$ arithmetic operations, which grows exponentially w.r.t. $K$ \cite{chazelle1993optimal}. However, the mentioned bound corresponds to a worst-case analysis which might not be the case in many real-world or randomly formed databases. As a result, we have faced a far smaller computational burden in our experiments, i.e. $g_{n,K}\ll \tilde{O}\left(n^{\left\lfloor K/2\right\rfloor}\right)$ even for moderate $K$. A detailed experimental or analytic investigation of this issue goes beyond the scope of this paper.

Therefore, the proposed strategy reduces the expected computational complexity of approximating the gradient to $O\left(g_{n,K} + f_{n,K}K^3\right)$, where $f_{n,K}$ denotes the expected number of vertices of the convex hull of $n$ randomly chosen points inside a $K$-simplex. It can be at least empirically confirmed that for large $n$, we have $f_{n,K}\ll n$. As a result of the mentioned initialization and acceleration strategies for moderately small $K$, our method runs fairly faster than its rival frameworks with a negligible harm to its performance. Detailed experiments on the performance of Algorithm \ref{alg:rcchgd} are given in the next sub-section. In all the following simulations we have always used the discussed initialization technique unless we say otherwise.

One more issue that should be discussed before proceeding to the next sub-section, is that finding the convex hull of a dataset $\boldsymbol{D}$ might be very sensitive to the presence of outliers in the dataset. However, analysis of the effect of outliers goes beyond the scope of this paper and we leave this interesting line of work for future researches in this area.
\subsection{Synthetic Data}
\label{sec:exp:synth}

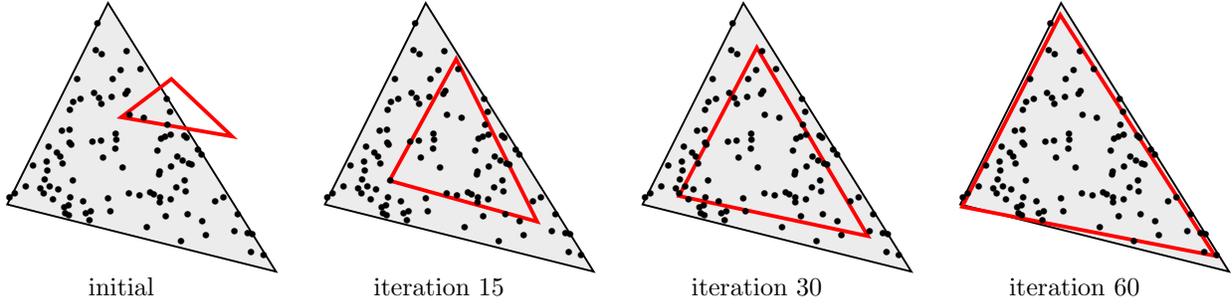
\begin{figure}[t]
\centering
\resizebox{\textwidth}{!}{
\begin{tikzpicture}
\draw[thick,fill=gray!15!white] (0,0) node[right=1.7cm,above=-1.5cm]{initial} -- (4,-1) -- (1.5,3) -- cycle;
\draw[ultra thick,color=red] (2.442932,1.868366) -- (1.694873,1.298958) -- (3.345374,1.014067) -- cycle;
\foreach \Point in{
(0.848794660917392,-0.0960087506414221),
(1.39974584193181,0.908704009802692),
(1.27105972616693,0.889150922992085),
(2.89678087742158,0.702200129909013),
(0.977373377387783,0.312873195862053),
(1.09489114517617,1.52526816672795),
(2.68665737540709,-0.195703680089585),
(0.927634131445215,1.07190187142706),
(0.74074871041492,0.722630127319107),
(1.794101041323,1.63593021820439),
(0.867288916099507,-0.179506202573398),
(1.40404438047537,1.45075089827069),
(0.938969125056707,0.0656107226888714),
(2.03591987357131,1.25033366694503),
(2.67394962633684,0.948271658428248),
(1.9882570485065,1.97325845819954),
(1.17226837682479,-0.18139782991532),
(2.23981112608354,0.757208528668067),
(1.76520162580043,1.60416660990923),
(0.635597283838206,0.217686061864303),
(0.587159432797263,0.595847947770355),
(1.71819753365937,0.498990385457372),
(0.54608000220669,0.116453175381386),
(1.50221642028552,0.0364928814325114),
(0.597736392751514,0.464647417832684),
(0.683326443149103,0.0355095888552021),
(0.921363358550676,-0.201269439209057),
(0.616706651288675,0.816219622953613),
(1.96894186565923,0.0936193651036874),
(0.913087171519253,-0.0613008727318241),
(2.58920961276213,-0.585990611709679),
(1.78086438246518,2.24136529573158),
(0.792868759401781,0.664809871307267),
(2.27395131104571,0.439104910794799),
(0.115919701462463,0.111571297140566),
(3.38113731686577,-0.445607467101213),
(1.53463205547825,-0.154268686021476),
(1.81447914634006,0.116741059556908),
(3.09624126187146,0.00299100644111672),
(2.90258225238036,-0.302655233840027),
(2.63673789637804,0.974607127841609),
(2.6153985486242,0.610813661683435),
(3.81605812889484,-0.79512001753746),
(1.36730420360817,1.53289182749692),
(1.04408923435303,1.81722545118905),
(2.28971097962216,0.926480386467958),
(2.27462620956845,-0.019604818253276),
(2.74769847323982,0.552412065059242),
(2.12989706861351,0.137428921315117),
(2.35372444910749,0.967676368541329),
(1.82153850203962,1.29033349788934),
(2.43412355099858,0.0257352561936723),
(2.83918935193604,0.764973386298183),
(3.43329771052554,-0.18840731144752),
(1.61224884064367,0.789939750599681),
(1.42906488195227,0.656850569385188),
(2.95983973712982,-0.506756708599542),
(0.983263842440874,1.48073386213297),
(2.42526742172866,1.32542905995435),
(1.61794540564703,1.00332590694761),
(1.24676164384208,-0.149277715266322),
(0.384433356811182,0.537289902777332),
(2.70543887344977,0.557324408242272),
(1.62493424322163,0.918468994657403),
(2.38751590635269,1.14154084849257),
(0.495611127892469,0.187548606899477),
(2.44128214750198,0.523242151607833),
(1.22443592772317,-0.28764476917127),
(1.57221284269901,1.82618524852756),
(0.556932784940584,0.326759276930003),
(0.827175763493381,0.567217438761972),
(3.57968777314085,-0.480144431779756),
(2.07896579704218,-0.392669953897953),
(2.6579001840711,0.146870282809987),
(0.931035250988445,0.868775939301168),
(2.37386272940182,1.52556776253711),
(0.038432371811898,0.0552108653125638),
(0.81570053219528,1.0508629221792),
(0.278508909619764,0.216969646383172),
(1.29319555758609,1.61283327793817),
(2.3589925914198,-0.228763208795102),
(2.18515439840576,0.104077249452846),
(2.52091662809705,0.212164172852418),
(2.4361362787148,0.994185388613279),
(0.952044296733873,0.845804589661893),
(0.761180940185605,0.170652110729708),
(0.934599210160287,1.36055498733788),
(1.6909373997113,1.96073644389913),
(1.34372776046531,2.65491643601027),
(1.13441993083987,-0.0039237619199744),
(3.65175464597701,-0.496229555694761),
(3.63226873376463,-0.759076042369825),
(2.64494889381075,0.315375860368441),
(1.32129511250454,2.25531532016346),
(2.26360530011743,0.406466294455832),
(2.53273569632627,0.669349516528993),
(1.55883171935907,1.5481096312368),
(3.2187250991002,0.0652935943890895),
(1.40910442640754,2.19641559450416),
(2.27635323028504,0.0552503216639951)
}{
    \node at \Point {\tiny\raisebox{1ex}{\textbullet}};
}
\end{tikzpicture}
\hspace*{3mm}
\begin{tikzpicture}
\draw[thick,fill=gray!15!white] (0,0) node[right=1.7cm,above=-1.5cm]{iteration $15$} -- (4,-1) -- (1.5,3) -- cycle;
\draw[ultra thick,color=red] (1.953519,2.166335) -- (0.9618828,0.3548715) -- (3.1591255,-0.2559511) -- cycle;
\foreach \Point in{
(0.848794660917392,-0.0960087506414221),
(1.39974584193181,0.908704009802692),
(1.27105972616693,0.889150922992085),
(2.89678087742158,0.702200129909013),
(0.977373377387783,0.312873195862053),
(1.09489114517617,1.52526816672795),
(2.68665737540709,-0.195703680089585),
(0.927634131445215,1.07190187142706),
(0.74074871041492,0.722630127319107),
(1.794101041323,1.63593021820439),
(0.867288916099507,-0.179506202573398),
(1.40404438047537,1.45075089827069),
(0.938969125056707,0.0656107226888714),
(2.03591987357131,1.25033366694503),
(2.67394962633684,0.948271658428248),
(1.9882570485065,1.97325845819954),
(1.17226837682479,-0.18139782991532),
(2.23981112608354,0.757208528668067),
(1.76520162580043,1.60416660990923),
(0.635597283838206,0.217686061864303),
(0.587159432797263,0.595847947770355),
(1.71819753365937,0.498990385457372),
(0.54608000220669,0.116453175381386),
(1.50221642028552,0.0364928814325114),
(0.597736392751514,0.464647417832684),
(0.683326443149103,0.0355095888552021),
(0.921363358550676,-0.201269439209057),
(0.616706651288675,0.816219622953613),
(1.96894186565923,0.0936193651036874),
(0.913087171519253,-0.0613008727318241),
(2.58920961276213,-0.585990611709679),
(1.78086438246518,2.24136529573158),
(0.792868759401781,0.664809871307267),
(2.27395131104571,0.439104910794799),
(0.115919701462463,0.111571297140566),
(3.38113731686577,-0.445607467101213),
(1.53463205547825,-0.154268686021476),
(1.81447914634006,0.116741059556908),
(3.09624126187146,0.00299100644111672),
(2.90258225238036,-0.302655233840027),
(2.63673789637804,0.974607127841609),
(2.6153985486242,0.610813661683435),
(3.81605812889484,-0.79512001753746),
(1.36730420360817,1.53289182749692),
(1.04408923435303,1.81722545118905),
(2.28971097962216,0.926480386467958),
(2.27462620956845,-0.019604818253276),
(2.74769847323982,0.552412065059242),
(2.12989706861351,0.137428921315117),
(2.35372444910749,0.967676368541329),
(1.82153850203962,1.29033349788934),
(2.43412355099858,0.0257352561936723),
(2.83918935193604,0.764973386298183),
(3.43329771052554,-0.18840731144752),
(1.61224884064367,0.789939750599681),
(1.42906488195227,0.656850569385188),
(2.95983973712982,-0.506756708599542),
(0.983263842440874,1.48073386213297),
(2.42526742172866,1.32542905995435),
(1.61794540564703,1.00332590694761),
(1.24676164384208,-0.149277715266322),
(0.384433356811182,0.537289902777332),
(2.70543887344977,0.557324408242272),
(1.62493424322163,0.918468994657403),
(2.38751590635269,1.14154084849257),
(0.495611127892469,0.187548606899477),
(2.44128214750198,0.523242151607833),
(1.22443592772317,-0.28764476917127),
(1.57221284269901,1.82618524852756),
(0.556932784940584,0.326759276930003),
(0.827175763493381,0.567217438761972),
(3.57968777314085,-0.480144431779756),
(2.07896579704218,-0.392669953897953),
(2.6579001840711,0.146870282809987),
(0.931035250988445,0.868775939301168),
(2.37386272940182,1.52556776253711),
(0.038432371811898,0.0552108653125638),
(0.81570053219528,1.0508629221792),
(0.278508909619764,0.216969646383172),
(1.29319555758609,1.61283327793817),
(2.3589925914198,-0.228763208795102),
(2.18515439840576,0.104077249452846),
(2.52091662809705,0.212164172852418),
(2.4361362787148,0.994185388613279),
(0.952044296733873,0.845804589661893),
(0.761180940185605,0.170652110729708),
(0.934599210160287,1.36055498733788),
(1.6909373997113,1.96073644389913),
(1.34372776046531,2.65491643601027),
(1.13441993083987,-0.0039237619199744),
(3.65175464597701,-0.496229555694761),
(3.63226873376463,-0.759076042369825),
(2.64494889381075,0.315375860368441),
(1.32129511250454,2.25531532016346),
(2.26360530011743,0.406466294455832),
(2.53273569632627,0.669349516528993),
(1.55883171935907,1.5481096312368),
(3.2187250991002,0.0652935943890895),
(1.40910442640754,2.19641559450416),
(2.27635323028504,0.0552503216639951)
}{
    \node at \Point {\tiny\raisebox{1ex}{\textbullet}};
}
\end{tikzpicture}
\hspace*{3mm}
\begin{tikzpicture}
\draw[thick,fill=gray!15!white] (0,0) node[right=1.7cm,above=-1.5cm]{iteration $30$} -- (4,-1) -- (1.5,3) -- cycle;
\draw[ultra thick,color=red] (1.700555,2.335072) -- (0.541803,0.126598) -- (3.3522756,-0.4704115) -- cycle;
\foreach \Point in{
(0.848794660917392,-0.0960087506414221),
(1.39974584193181,0.908704009802692),
(1.27105972616693,0.889150922992085),
(2.89678087742158,0.702200129909013),
(0.977373377387783,0.312873195862053),
(1.09489114517617,1.52526816672795),
(2.68665737540709,-0.195703680089585),
(0.927634131445215,1.07190187142706),
(0.74074871041492,0.722630127319107),
(1.794101041323,1.63593021820439),
(0.867288916099507,-0.179506202573398),
(1.40404438047537,1.45075089827069),
(0.938969125056707,0.0656107226888714),
(2.03591987357131,1.25033366694503),
(2.67394962633684,0.948271658428248),
(1.9882570485065,1.97325845819954),
(1.17226837682479,-0.18139782991532),
(2.23981112608354,0.757208528668067),
(1.76520162580043,1.60416660990923),
(0.635597283838206,0.217686061864303),
(0.587159432797263,0.595847947770355),
(1.71819753365937,0.498990385457372),
(0.54608000220669,0.116453175381386),
(1.50221642028552,0.0364928814325114),
(0.597736392751514,0.464647417832684),
(0.683326443149103,0.0355095888552021),
(0.921363358550676,-0.201269439209057),
(0.616706651288675,0.816219622953613),
(1.96894186565923,0.0936193651036874),
(0.913087171519253,-0.0613008727318241),
(2.58920961276213,-0.585990611709679),
(1.78086438246518,2.24136529573158),
(0.792868759401781,0.664809871307267),
(2.27395131104571,0.439104910794799),
(0.115919701462463,0.111571297140566),
(3.38113731686577,-0.445607467101213),
(1.53463205547825,-0.154268686021476),
(1.81447914634006,0.116741059556908),
(3.09624126187146,0.00299100644111672),
(2.90258225238036,-0.302655233840027),
(2.63673789637804,0.974607127841609),
(2.6153985486242,0.610813661683435),
(3.81605812889484,-0.79512001753746),
(1.36730420360817,1.53289182749692),
(1.04408923435303,1.81722545118905),
(2.28971097962216,0.926480386467958),
(2.27462620956845,-0.019604818253276),
(2.74769847323982,0.552412065059242),
(2.12989706861351,0.137428921315117),
(2.35372444910749,0.967676368541329),
(1.82153850203962,1.29033349788934),
(2.43412355099858,0.0257352561936723),
(2.83918935193604,0.764973386298183),
(3.43329771052554,-0.18840731144752),
(1.61224884064367,0.789939750599681),
(1.42906488195227,0.656850569385188),
(2.95983973712982,-0.506756708599542),
(0.983263842440874,1.48073386213297),
(2.42526742172866,1.32542905995435),
(1.61794540564703,1.00332590694761),
(1.24676164384208,-0.149277715266322),
(0.384433356811182,0.537289902777332),
(2.70543887344977,0.557324408242272),
(1.62493424322163,0.918468994657403),
(2.38751590635269,1.14154084849257),
(0.495611127892469,0.187548606899477),
(2.44128214750198,0.523242151607833),
(1.22443592772317,-0.28764476917127),
(1.57221284269901,1.82618524852756),
(0.556932784940584,0.326759276930003),
(0.827175763493381,0.567217438761972),
(3.57968777314085,-0.480144431779756),
(2.07896579704218,-0.392669953897953),
(2.6579001840711,0.146870282809987),
(0.931035250988445,0.868775939301168),
(2.37386272940182,1.52556776253711),
(0.038432371811898,0.0552108653125638),
(0.81570053219528,1.0508629221792),
(0.278508909619764,0.216969646383172),
(1.29319555758609,1.61283327793817),
(2.3589925914198,-0.228763208795102),
(2.18515439840576,0.104077249452846),
(2.52091662809705,0.212164172852418),
(2.4361362787148,0.994185388613279),
(0.952044296733873,0.845804589661893),
(0.761180940185605,0.170652110729708),
(0.934599210160287,1.36055498733788),
(1.6909373997113,1.96073644389913),
(1.34372776046531,2.65491643601027),
(1.13441993083987,-0.0039237619199744),
(3.65175464597701,-0.496229555694761),
(3.63226873376463,-0.759076042369825),
(2.64494889381075,0.315375860368441),
(1.32129511250454,2.25531532016346),
(2.26360530011743,0.406466294455832),
(2.53273569632627,0.669349516528993),
(1.55883171935907,1.5481096312368),
(3.2187250991002,0.0652935943890895),
(1.40910442640754,2.19641559450416),
(2.27635323028504,0.0552503216639951)
}{
    \node at \Point {\tiny\raisebox{1ex}{\textbullet}};
}
\end{tikzpicture}
\hspace*{3mm}
\begin{tikzpicture}
\draw[thick,fill=gray!15!white] (0,0) node[right=1.7cm,above=-1.5cm]{iteration $60$} -- (4,-1) -- (1.5,3) -- cycle;
\draw[ultra thick,color=red] (1.488589,2.823049) -- (0.02855040,-0.03387861) -- (3.7912914,-0.7654253) -- cycle;
\foreach \Point in{
(0.848794660917392,-0.0960087506414221),
(1.39974584193181,0.908704009802692),
(1.27105972616693,0.889150922992085),
(2.89678087742158,0.702200129909013),
(0.977373377387783,0.312873195862053),
(1.09489114517617,1.52526816672795),
(2.68665737540709,-0.195703680089585),
(0.927634131445215,1.07190187142706),
(0.74074871041492,0.722630127319107),
(1.794101041323,1.63593021820439),
(0.867288916099507,-0.179506202573398),
(1.40404438047537,1.45075089827069),
(0.938969125056707,0.0656107226888714),
(2.03591987357131,1.25033366694503),
(2.67394962633684,0.948271658428248),
(1.9882570485065,1.97325845819954),
(1.17226837682479,-0.18139782991532),
(2.23981112608354,0.757208528668067),
(1.76520162580043,1.60416660990923),
(0.635597283838206,0.217686061864303),
(0.587159432797263,0.595847947770355),
(1.71819753365937,0.498990385457372),
(0.54608000220669,0.116453175381386),
(1.50221642028552,0.0364928814325114),
(0.597736392751514,0.464647417832684),
(0.683326443149103,0.0355095888552021),
(0.921363358550676,-0.201269439209057),
(0.616706651288675,0.816219622953613),
(1.96894186565923,0.0936193651036874),
(0.913087171519253,-0.0613008727318241),
(2.58920961276213,-0.585990611709679),
(1.78086438246518,2.24136529573158),
(0.792868759401781,0.664809871307267),
(2.27395131104571,0.439104910794799),
(0.115919701462463,0.111571297140566),
(3.38113731686577,-0.445607467101213),
(1.53463205547825,-0.154268686021476),
(1.81447914634006,0.116741059556908),
(3.09624126187146,0.00299100644111672),
(2.90258225238036,-0.302655233840027),
(2.63673789637804,0.974607127841609),
(2.6153985486242,0.610813661683435),
(3.81605812889484,-0.79512001753746),
(1.36730420360817,1.53289182749692),
(1.04408923435303,1.81722545118905),
(2.28971097962216,0.926480386467958),
(2.27462620956845,-0.019604818253276),
(2.74769847323982,0.552412065059242),
(2.12989706861351,0.137428921315117),
(2.35372444910749,0.967676368541329),
(1.82153850203962,1.29033349788934),
(2.43412355099858,0.0257352561936723),
(2.83918935193604,0.764973386298183),
(3.43329771052554,-0.18840731144752),
(1.61224884064367,0.789939750599681),
(1.42906488195227,0.656850569385188),
(2.95983973712982,-0.506756708599542),
(0.983263842440874,1.48073386213297),
(2.42526742172866,1.32542905995435),
(1.61794540564703,1.00332590694761),
(1.24676164384208,-0.149277715266322),
(0.384433356811182,0.537289902777332),
(2.70543887344977,0.557324408242272),
(1.62493424322163,0.918468994657403),
(2.38751590635269,1.14154084849257),
(0.495611127892469,0.187548606899477),
(2.44128214750198,0.523242151607833),
(1.22443592772317,-0.28764476917127),
(1.57221284269901,1.82618524852756),
(0.556932784940584,0.326759276930003),
(0.827175763493381,0.567217438761972),
(3.57968777314085,-0.480144431779756),
(2.07896579704218,-0.392669953897953),
(2.6579001840711,0.146870282809987),
(0.931035250988445,0.868775939301168),
(2.37386272940182,1.52556776253711),
(0.038432371811898,0.0552108653125638),
(0.81570053219528,1.0508629221792),
(0.278508909619764,0.216969646383172),
(1.29319555758609,1.61283327793817),
(2.3589925914198,-0.228763208795102),
(2.18515439840576,0.104077249452846),
(2.52091662809705,0.212164172852418),
(2.4361362787148,0.994185388613279),
(0.952044296733873,0.845804589661893),
(0.761180940185605,0.170652110729708),
(0.934599210160287,1.36055498733788),
(1.6909373997113,1.96073644389913),
(1.34372776046531,2.65491643601027),
(1.13441993083987,-0.0039237619199744),
(3.65175464597701,-0.496229555694761),
(3.63226873376463,-0.759076042369825),
(2.64494889381075,0.315375860368441),
(1.32129511250454,2.25531532016346),
(2.26360530011743,0.406466294455832),
(2.53273569632627,0.669349516528993),
(1.55883171935907,1.5481096312368),
(3.2187250991002,0.0652935943890895),
(1.40910442640754,2.19641559450416),
(2.27635323028504,0.0552503216639951)
}{
    \node at \Point {\tiny\raisebox{1ex}{\textbullet}};
}
\end{tikzpicture}}
\caption{\small{Snapshots from running Algorithm \ref{alg:rcchgd} on a set of $n=100$ noiseless samples drawn uniformly from a two-dimensional simplex. The original triangle is drawn in black and the outputs of the proposed method for four different iteration steps are shown in red.}}
\label{fig:Snapshot}
\end{figure}

\begin{SCfigure}[1][b]
\centering
\includegraphics[width=0.5\linewidth,clip]{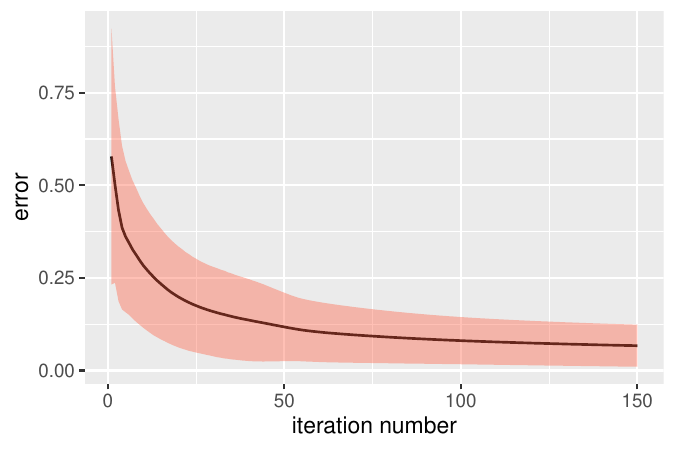}
\caption{\small{Depiction of $error$ in \eqref{eq:errorMeasure} as a function of iteration number for Algorithm \ref{alg:rcchgd}. The experiment has been performed on $n=100$ data points uniformly sampled from a two-dimensional simplex. Parameters $\gamma$ and optimization step $\alpha$ have been adjusted to optimize the performance. According to the curve, sample mean and the standard deviation of $error$ decay as the number of iterations is increased.}}
\label{fig:iteration_error}
\end{SCfigure}


Figure \ref{fig:Snapshot} shows four snapshots from the outputs of Algorithm \ref{alg:rcchgd} running on a set of $n=100$ noiseless data points in $\mathbb{R}^2$, which have been uniformly sampled from a typical two-dimensional simplex. We have chosen $K=2$ in order to facilitate the visualization of our results. For this particular simulation, we have chosen a completely random initial simplex instead of the suggested technique in the previous sub-section. The reason behind this choice is to show the convergence performance of Algorithm \ref{alg:rcchgd} and the fact that it does not necessarily need a good initialization. Here, the shaded simplex denotes the true $\mathcal{S}_T$,  while the ``red" simplex is the output of Algorithm \ref{alg:rcchgd} for four different iteration numbers. As it can be seen, despite the fact that algorithm has started from an inappropriate initial point, it converges to an acceptable solution in a reasonable number of iterations. Figure \ref{fig:iteration_error} depicts {\it {error}} as a function of iteration number for the experiment of Figure \ref{fig:Snapshot}. Results have been averaged over $100$ trials. According to the curve, sample mean and the standard deviation of $error$ decay as the number of iterations increases.

\begin{figure}[t]
\centering
        \begin{subfigure}[b]{0.5\textwidth}
                \includegraphics[trim=0mm 2mm 11mm 3mm,clip,width=\textwidth]{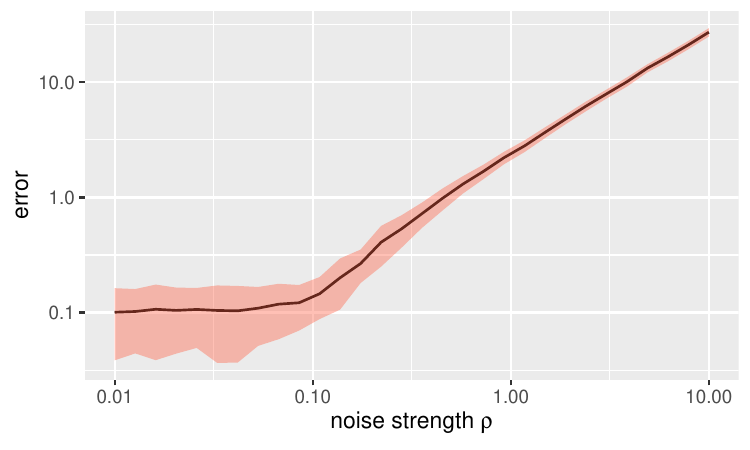}
                \caption{}
                \label{fig:noise}
        \end{subfigure}
        \hspace*{2mm}
        \begin{subfigure}[b]{0.45\textwidth}
                \includegraphics[trim=0mm 2mm 11mm 3mm,clip,width=\textwidth]{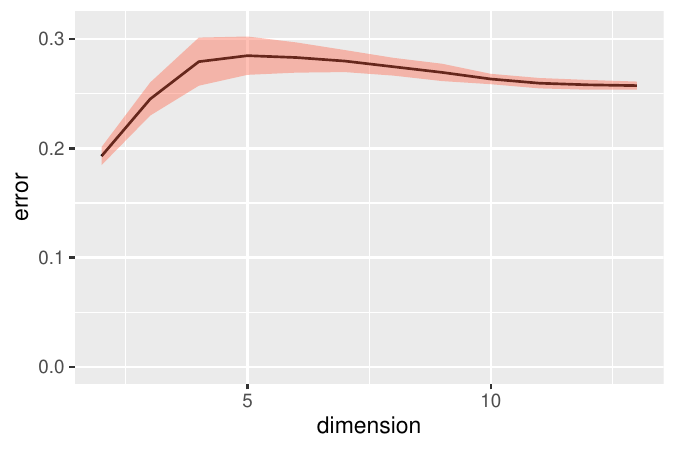}
                \caption{}
                \label{fig:moshkel}
        \end{subfigure}
        \caption{\small{Estimation $error$ as a function of noise strength $\rho$ and dimension. In \ref{fig:noise}, $n=100$ data samples are drawn from a two-dimensional simplex and then contaminated with additive white Gaussian noise. However, for \ref{fig:moshkel} data samples are noiseless and $n$ has been increased proportional to $K^2\log K$, where $K$ indicates the dimension.}}
\label{fig:NoiseDim}
\end{figure}

Figure \ref{fig:NoiseDim} aims to investigate the robustness of our method to noise and also the dimension of the space. We have used synthetic datasets for this purpose. In particular, Figure \ref{fig:noise} illustrates the performance of Algorithm \ref{alg:rcchgd} when samples are contaminated with additive white Gaussian noise. In this regard, $n=100$ uniformly-sampled data points have been drawn from a two-dimensional simplex. Each sample has been contaminated with a noise vector drawn independently from a zero-mean Gaussian distribution  $\mathcal{N}\left(\boldsymbol{0},\sigma^2\boldsymbol{I}\right)$, where $\sigma$ denotes the standard deviation of the noise components. Let us define the {\it {noise strength}} $\rho$ as the normalized standard deviation of noise as follows:
\begin{equation}
\rho\triangleq\sigma\left[\frac{1}{K\left(K+1\right)}\sum_{k,k'}\left\Vert\boldsymbol{\theta}_k-\boldsymbol{\theta}_{k'}\right\Vert_2\right]^{-1}.
\end{equation}
According to the above definition, value of $\rho$ is invariant to the size of the original simplex $\mathcal{S}_T$. Figure \ref{fig:noise} depicts the $error$ of Algorithm \ref{alg:rcchgd} as a function of noise strength $\rho$. The proposed method is considerably robust to noise roughly for $\rho\leq 0.2$. However, the performance is degraded when noise strength is being increased. On the other hand, Figure \ref{fig:moshkel} illustrates the $error$ as a function of dimension $K$. For each point of the curve, dataset size has been increased proportional to $K^2\log K$ in order to match the result of Theorem \ref{thm:PACnoiseless} and Corollaries \ref{corl:ML} and \ref{corl:softML}. Surprisingly, we have empirically validated that as long as $n$ is increased proportional to $K^2\log K$, the $error$ remains constant. This implies that Algorithm \ref{alg:rcchgd} has the same sample complexity (order-wise) as that of ML, while unlike the ML estimator it converges to a solution in a practically reasonable time. 

\begin{SCtable}[1.5][b]
  \centering
  \begin{tabular}{|c|c|c|c|}
    \cline{2-4}
    \multicolumn{1}{c|}{} & Plain & Noisy & HD \\ \hline
    Proposed & $0.20$   & $\mathbf{0.51}$   & $0.74$  \\ \hline
    MVSA     & $0.14$   & $1.84$   & $0.76$  \\ \hline
    SISAL    & $0.16$   & $1.65$   & $0.77$  \\ \hline
    VCA      & $1.09$   & $1.006$  & $5.93$  \\ \hline
    UNMIX    & $0.14$   & $1.83$   & -       \\ \hline
  \end{tabular}
  \hspace*{2mm}
\caption{\small{Comparison of the proposed method with MVSA \cite{li2015minimum}, SISAL \cite{bioucas2009variable}, VCA \cite{nascimento2005vertex}, and UNMIX \cite{tolliver2010robust}. Methods have been tested on three different datasets. The values of $error$ have been averaged over several runs, such that all relative standard deviations become less than $10\%$. UNMIX did not execute on ``HD" dataset in a reasonable time.}}
\label{tab:comp}
\end{SCtable}

In Table \ref{tab:comp}, we have compared the $error$ performance of our method with a number of well-known computational techniques for learning of simplices, namely MVSA \cite{li2015minimum}, SISAL \cite{bioucas2009variable}, VCA \cite{nascimento2005vertex}, and UNMIX \cite{tolliver2010robust}. We have used three datasets for the experiments which are i) Plain dataset: $n=100$ data points generated from a simple two-dimensional simplex, ii) Noisy dataset: a noisy version of the ``Plain" dataset with $\rho=0.5$ (highly noisy), and iii) High-Dimensional (HD) dataset: $n=1000$ data points generated uniformly from a noiseless simplex with $K=9$. According to Table \ref{tab:comp}, our method has a comparable performance on ``Plain" and ``HD" datasets, while it outperforms all the rival strategies on Noisy dataset. The reason behind this result might be due to the continuously-relaxed nature of our method which is in contrast to the majority of existing frameworks in this area. In fact, the inherent flexibility of the objective function in \eqref{eq:proposed} achieves more robustness to noisy data points that fall out of the true simplex.

\subsection{Computational Biology}
\label{sec:exp:compbio}

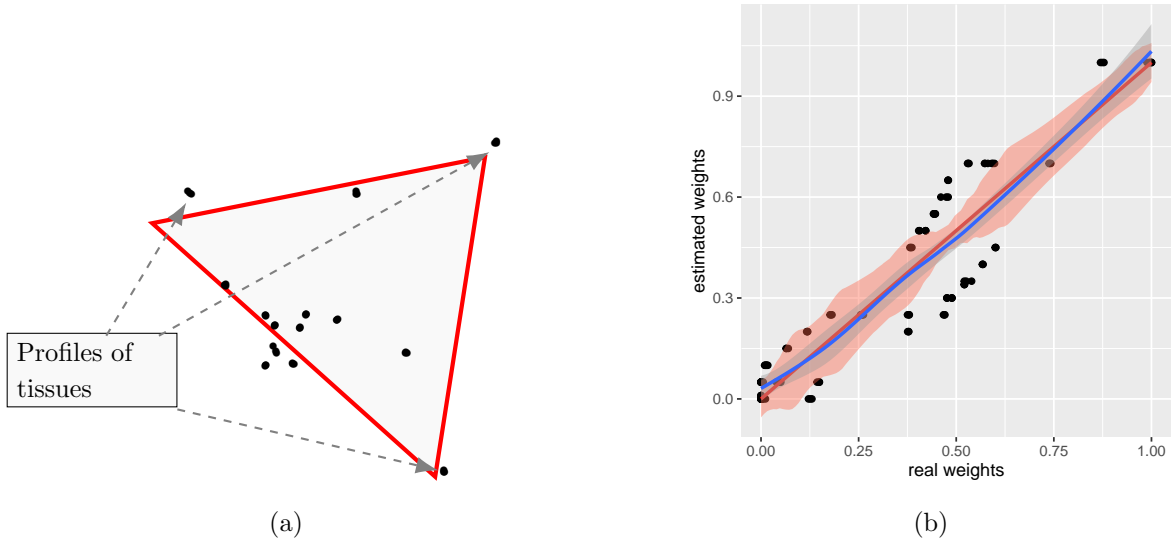
\begin{figure}[t]
\centering
        \begin{subfigure}[b]{0.45\textwidth}
                \begin{tikzpicture}[xscale=7,yscale=7]
\draw[ultra thick,color=red,fill=gray!5!white] (0.2246169,-0.3129046) -- (-0.3127093,0.1688291) node[right=-8mm,above=-2.45cm,color=black,align=left,text width=2cm,draw,thin,fill=gray!5!white]{\small{Profiles of\\tissues}} -- (0.3192377,0.2942368) -- cycle;
\draw[->,thick,style=dashed,-{Latex[length=3mm]},color=gray] (-0.4,-0.045) -- (-0.247837128813333,0.208282727471155);
\draw[->,thick,style=dashed,-{Latex[length=3mm]},color=gray] (-0.3,-0.045) -- (0.328131848336619,0.304260805690042);
\draw[->,thick,style=dashed,-{Latex[length=3mm]},color=gray] (-0.265,-0.185) -- (0.231062345890597,-0.30090565736794);
\foreach \Point in{
(0.241062345890597,-0.31090565736794),
(0.241492346391997,-0.306491373938237),
(0.242204302198292,-0.31028130189314),
(-0.237837128813333,0.218282727471155),
(-0.241242878449203,0.219828395421965),
(-0.244717972722455,0.223637960956324),
(0.341113674380515,0.314379323098166),
(0.338131848336619,0.314260805690042),
(0.341499626184849,0.318989333706044),
(0.0746726878614655,0.223236144919957),
(0.0744571911720449,0.217547706129814),
(0.0765208069229126,0.216759386301708),
(0.168384617160646,-0.0838239549216155),
(0.168858993295121,-0.0818920266535703),
(0.171822390670672,-0.0856662096053728),
(-0.172513271449556,0.0424310712207473),
(-0.172661728985868,0.0469317045430077),
(-0.175226818671407,0.0444431424359086),
(-0.0459411346271626,-0.104526106062333),
(-0.0452738266824777,-0.105492684945562),
(-0.0432477305975817,-0.105555367407904),
(-0.0971851167224968,-0.0121856114833435),
(-0.097444967637554,-0.0134324918561238),
(-0.0966080025213357,-0.0150597381552873),
(0.0396649081847912,-0.0207614388963999),
(0.0379238810481281,-0.0226229597868649),
(0.0378914942118878,-0.0211007261738968),
(-0.0208648187449568,-0.0115433764363155),
(-0.0204721680946645,-0.0100229954956309),
(-0.0201057214871735,-0.0110018485618566),
(-0.0323186281393812,-0.0374724485927024),
(-0.0324332913586788,-0.0373606458658648),
(-0.0315256681002676,-0.0344549086952489),
(-0.0798576274924564,-0.0321748155916029),
(-0.0805287350272375,-0.0326989349979263),
(-0.0791673020133657,-0.0321021986272637),
(-0.0778616377179512,-0.081107566387786),
(-0.0763809047751631,-0.085015599304082),
(-0.0829762841087377,-0.071995320000117),
(-0.0966256368294811,-0.108252010058871),
(-0.0965538740248481,-0.106798445277303),
(-0.0981282381157371,-0.108928938854668)}{
    \node at \Point {\tiny\raisebox{1ex}{\textbullet}};
}
\end{tikzpicture}
                \caption{}
                \label{fig:tibSnap}
        \end{subfigure}
        \hspace*{13mm}
        \begin{subfigure}[b]{0.4\textwidth}
                \includegraphics[clip,width=\textwidth]{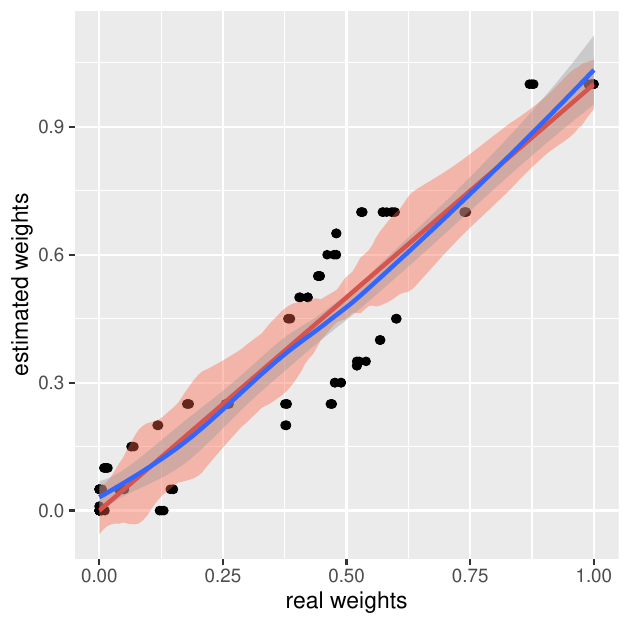}
                \caption{}
                \label{fig:tinCurve}
        \end{subfigure}%
        \caption{\small{Cell-type identification from micro-array data given in \cite{shen2010cell}. \ref{fig:tibSnap}: Visualization of data points, as well as the estimated simplex. Vertices of the estimated simplex highly resemble the expression levels of the ground truth tissues. \ref{fig:tinCurve}: Estimated weights for the samples as a function of real weights reported in the dataset. Data points are scattered around the $X=Y$ curve (red). Also, the result of a LOESS regression of the samples (blue) falls very close to the $X=Y$ curve.}}
\label{fig:compbi}
\end{figure}

Unmixing of cell-types from biological samples has become a commonplace task in bioinformatics, and in particular, in cancer therapy. In a recent work \cite{shen2010cell}, researchers have synthetically merged cells from three different tissues (brain, liver and kidney) with several different percentages, where combination weights were known prior to each merging. The aim of this experiment is to investigate whether one can deconvolve the gene-expression profiles of different tissues from a synthetic bulk dataset. However, the dataset in \cite{shen2010cell} is also a potentially appropriate target for Algorithm \ref{alg:rcchgd}. In fact, we can neglect the knowledge of the combination weights, and try to blindly infer both the unknown gene-expression profiles of the three tissues, as well as the combination weights, directly from the gene-expression levels of samples.

In this regard, we have used all the $n=42$ samples of the dataset in \cite{shen2010cell}, where the initial dimensionality of data is $K_{\mathrm{init}}=31100$. Due to the presence of redundant information in the dataset\footnote{Any $n$ points in $\mathbb{R}^{K_{\mathrm{init}}}$, for $n\leq K_{\mathrm{init}}$, lie on a $\left(n-1\right)$-dimensional linear subspace which is almost surely identifiable and can be attained by Principal Component Analysis. This way, the reduction of dimensionality in this case does not result into any information loss.}, samples have been linearly projected onto a corresponding $41$-dimensional subspace via Principal Component Analysis (PCA), and then the proposed method has been applied. Figure \ref{fig:compbi} illustrates the final results, where Figure \ref{fig:tibSnap} shows a qualitative two-dimensional visualization of the final simplex (via the first two principal components) that matches the ground truth gene-expression profiles of the tissues with a relatively high accuracy. From a more quantitative perspective, Figure \ref{fig:tinCurve} shows the inferred weights for all the $n=42$ samples as a function of real weights which are already given in the dataset. As can be seen, points are scattered around the $X=Y$ curve which is depicted in red. We have also shown the result of a LOESS (Locally Estimated Scatterplot Smoothing) regression of the samples via a blue curve, which is very close to the $X=Y$ curve. This suggests that the real weights and the ones that are estimated via our proposed method are similar to each other.


\begin{figure}[t]
\centering
        \begin{subfigure}[b]{0.50\textwidth}
                \includegraphics[trim=40mm 100mm 30mm 110mm,clip,width=\textwidth]{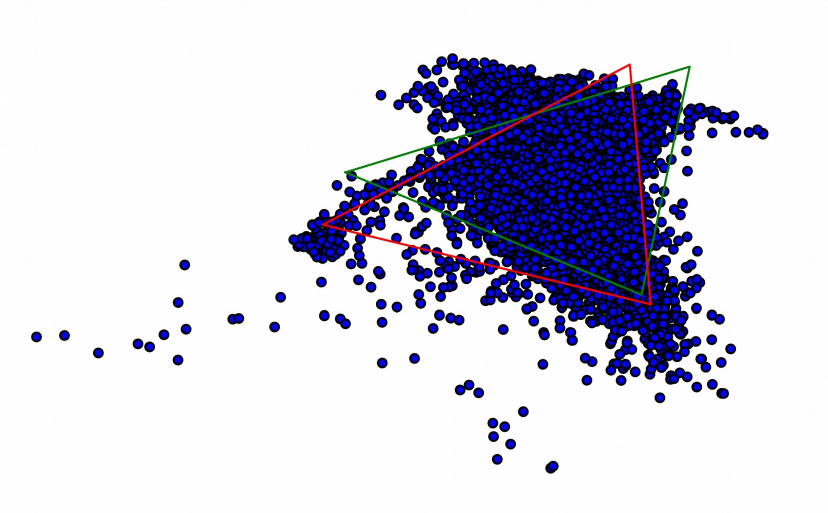}
                \caption{}
                \label{fig:cupritSnap}
        \end{subfigure}
        \begin{subfigure}[b]{0.48\textwidth}
                \includegraphics[trim=0mm 0mm 0mm 0mm,clip,width=\textwidth]{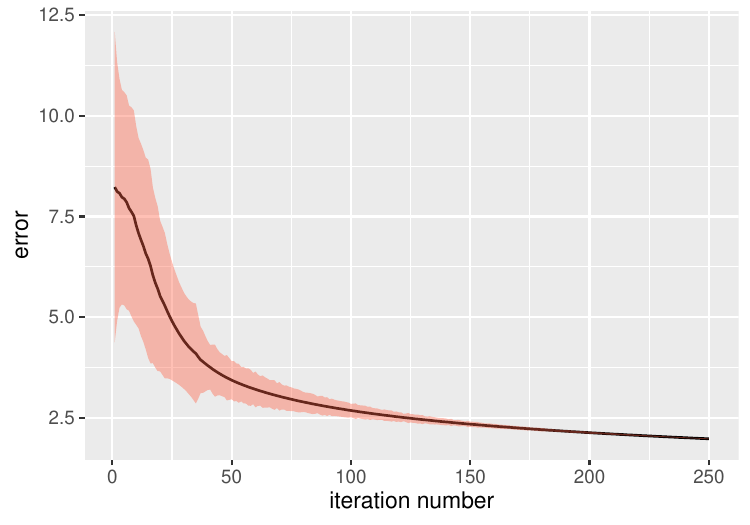}
                \caption{}
                \label{fig:cupritCurve}
        \end{subfigure}%
        \caption{\small{(a): Estimation of dominant elements in a hyper-spectral imaging dataset (red simplex) compared to the ground truth (green simplex). Three major elements, called {\it {Andradite}}, {\it {Dumortierite}} and {\it {Sphene}} have been identified which considerably match with our experimental findings. Data points have been visualized via the first two PCs. (b): Estimation error as a function of iterations for the {\it {Cuprit}} dataset.}}
\label{fig:cuprit}
\end{figure}

\subsection{Hyper-Spectral Remote Sensing}
\label{sec:exp:nasa}
Another major application of the ``unmixing problem" is in hyper-spectral remote sensing, where one aims to reveal the chemical composition of materials in a remote area by using hyper-spectral imaging devices. Each pixel in a hyper-spectral image is a vector including radiation intensities (in different wave-lengths) from a relatively small region of a remote location. This way, pixels can be considered as weighted averages of a limited number of radiation profiles which correspond to the chemical elements that are likely to be found in that location. The {\it {Cuprit}} dataset \cite{fyzhu_2014_TIP_DgS_NMF} includes $n=47500$ pixels, where each pixel has been measured in effectively $m=188$ spectral wavelengths. In \cite{fyzhu_2014_TIP_DgS_NMF}, it has been suggested that $K=12$ basic elements might be present in the remote location that corresponds to the image pixels, however, a smaller subset of them might be dominant.

We have randomly chosen $n=10000$ points from this dataset and reduced the dimensionality of data to $\mathrm{dim}=10$ via PCA. The latter is performed in order to eliminate the majority of linear dependencies and also to suppress the effect of noise. According to a simple eigen-analysis of the data, we chose the number of vertices for the simplex as $K+1=3$. Based on the above setting, Figure \ref{fig:cuprit} illustrates the performance of the proposed method on the data compared to the ground truth profiles presented in the literature. Similar to the previous sub-section, Figure \ref{fig:cupritSnap} depicts a two-dimensional visualization, obtained from the first two principal components of the data samples and the estimated simplex via Algorithm \ref{alg:rcchgd} (shown in red). For the sake of comparison, the ground truth profiles have been shown by a green simplex. Evidently, our method finds the dominant elements with an acceptable accuracy. Figure \ref{fig:cupritCurve} shows the approximation $error$ versus the number of iterations for Algorithm \ref{alg:rcchgd}.

\section{Conclusions}
\label{sec:conc}

This paper aims to develop a computational framework for learning of simplices in arbitrary dimensions. In order to deal with the combinatoric structure of this problem, a continuously-relaxed optimization scheme has been proposed which comes with new theoretical sample complexity bounds. In particular, we have shown that a sufficient sample complexity for both ML and the proposed Soft-ML surrogate are $n\ge O\left(K^2\log K\right)$, which is a significant improvement of the existing bounds for this problem, i.e. $O\left(K^{22}\right)$. In addition, a heuristic, yet computationally tractable algorithm is proposed for numerical optimization of our scheme. Numerous experiments magnify the applicability of our method on synthetic and real-world datasets. Our method has shown a comparable performance to a number of well-known rival strategies on noiseless data, while shows a considerable superiority in noisy regimes. For future works, one may attempt to provide similar sample complexity bounds for the noisy case, which has not been tackled yet. Assuming non-uniform distributions over the simplices to generate the data could be another interesting problem to consider. Also, deriving  the computational complexity of this problem, as well as finding lower bounds on the sample complexity, are examples of other existing research directions in this area.

\bibliographystyle{IEEEtran}
\bibliography{IEEEabrv,ref}

\appendix
\section{Proof of Theorems}
\label{appendix:Theorems}

\begin{proof}[Proof of Theorem \ref{corl:ML}]
We show that the problem of estimating the true simplex $\mathcal{S}_T$ up to a total-variation error of $\epsilon$ can be transformed into a binary classification task, where one aims to learn a classifier up to a same classification error rate. Mathematically speaking, for $\mathcal{S}\in\mathbb{S}_K$ and $\boldsymbol{X}\in\mathbb{R}^K$, let us define the parametric binary classifier $c\left(\cdot;\mathcal{S}\right)$ as
\begin{equation}
c\left(\boldsymbol{X};\mathcal{S}\right)\triangleq\left\{\begin{array}{lc}
0 & \boldsymbol{X}\in\mathcal{S}  \\
1 & \mathrm{O.W.}
\end{array}\right.,
\end{equation}
and the corresponding classifier family $\mathscr{C}$ as
\begin{equation}
\mathscr{C}\triangleq\left\{
c\left(\cdot;\mathcal{S}\right):\mathbb{R}^K\rightarrow\left\{1,0\right\}\bigg\vert~\mathcal{S}\in\mathbb{S}_K
\right\}.
\label{eq:defmathscrC}
\end{equation}
Having the input dataset $\boldsymbol{D}=\left\{\boldsymbol{X}_1,\ldots,\boldsymbol{X}_n\right\}$ drawn independently from distribution $\mathbb{P}_{\mathcal{S}_T}$, let us consider an augmented supervised dataset $\tilde{\boldsymbol{D}}$ which includes the $n$ data points $\boldsymbol{X}_1,\ldots,\boldsymbol{X}_n$, all labeled with $c\left(\cdot;\mathcal{S}_T\right)$; Thus, all data points in $\tilde{\boldsymbol{D}}$ are labeled with $0$. Assume there exists algorithm $\mathscr{A}$, which given the dataset $\tilde{\boldsymbol{D}}$, is able to output a {\it {consistent}} classifier $c^*\in\mathscr{C}$ and a corresponding simplex $\mathcal{S}^*\in\mathbb{S}_K$, i.e., $c^*\left(\cdot\right)=c\left(\cdot;\mathcal{S}^*\right)$, where the volume of $\mathcal{S}^*$ is strictly bounded. Here, by consistent, we mean $c^*$ classifies all the data points in $\tilde{\boldsymbol{D}}$ correctly. Mathematically speaking, let us assume $c^*$ and $\mathcal{S}^*$ have the following two properties:
\begin{align*}
&\mathrm{Consistency}:&\quad
c^*\left(\boldsymbol{X}_i\right)=0,~\mathrm{for~all}~i=1,\ldots,n,
\\
&\mathrm{Bounded~Volume~Property}:&\quad
\mathrm{Vol}\left(\mathcal{S}^*\right)\leq\mathrm{Vol}\left(\mathcal{S}_T\right),
\end{align*}
where $\mathcal{S}_T$ denotes the true simplex underlying the input data samples. We first claim that the simplex $\mathcal{S}^*$ which is found by algorithm $\mathscr{A}$ and has the above properties approximates the true simplex $\mathcal{S}_T$ up to a tolerable total-variation error. Then, we show that the ML algorithm has both of the above-mentioned properties and thus MLE is a desired estimator in the total-variation (Nikodym distance) sense. In order to do so, first we show the following:
\begin{lemma}
\label{lemma:TVlemma}
For any $\epsilon\in\left[0,1\right]$, if we have
$
\mathbb{E}_{\boldsymbol{X}\sim\mathbb{P}_{\mathcal{S}_T}}\left\{c^*\left(\boldsymbol{X}\right)\right\}\leq\epsilon,
$
then
$
\mathcal{D}_{\mathrm{TV}}\left(\mathbb{P}_{\mathcal{S}_T},\mathbb{P}_{\mathcal{S}^*}\right)\leq \epsilon
$.
\end{lemma}
\begin{proof}
It should be noted that
\begin{equation}
\mathbb{E}_{\boldsymbol{X}\sim\mathbb{P}_{\mathcal{S}_T}}\left\{c\left(\boldsymbol{X};\mathcal{S}^*\right)\right\}=
\mathbb{P}\left\{\boldsymbol{X}\sim\mathbb{P}_{\mathcal{S}_T}~\mathrm{and}~\boldsymbol{X}\in\mathcal{S}^*\right\}
\leq\epsilon, 
\end{equation}
which means $\mathrm{Vol}\left(\mathcal{S}^*\cap\mathcal{S}_T\right)\ge V_T\left(1-\epsilon\right)$, where $V_T=\mathrm{Vol}\left(\mathcal{S}_T\right)$. Also, we already know that $\mathrm{Vol}\left(\mathcal{S}^*\right)\leq V_T$. These facts lead us to the following three inequalities:
\begin{align}
&\mathrm{Vol}\left(\mathcal{S}^*\right)~\ge~ V_T\left(1-\epsilon\right),\\ &\mathrm{Vol}\left(\mathcal{S}^*-\mathcal{S}_T\right)~\leq~ \mathrm{Vol}\left(\mathcal{S}^*\right)-V_T\left(1-\epsilon\right)
\nonumber
\\
\mathrm{and}~~&
\mathrm{Vol}\left(\mathcal{S}_T-\mathcal{S}^*\right)~\leq~ V_T\epsilon.
\nonumber
\end{align}
On the other hand, based on the definition of the total variation distance between $\mathbb{P}_{\mathcal{S}^*}$ and $\mathbb{P}_{\mathcal{S}_T}$, we have
\begin{align}
\mathcal{D}_{\mathrm{TV}}\left(
\mathbb{P}_{\mathcal{S}^*},\mathbb{P}_{\mathcal{S}_T}
\right)
&\triangleq
\frac{1}{2}\left[
\int_{\mathcal{S}^*-\mathcal{S}_T}\rho_{\mathcal{S}^*}
+
\int_{\mathcal{S}_T-\mathcal{S}^*}\rho_{\mathcal{S}_T}
+
\int_{\mathcal{S}^*\cap\mathcal{S}_T}\left\vert
\rho_{\mathcal{S}^*}-\rho_{\mathcal{S}_T}
\right\vert
\right]
\nonumber\\
&\leq
\frac{1}{2}\left[
\frac{
\mathrm{Vol}\left(\mathcal{S}^*\right)-V_T\left(1-\epsilon\right)
}{\mathrm{Vol}\left(\mathcal{S}^*\right)}
+
\epsilon\cdot\frac{V_T}{V_T}
+
\left(1-\epsilon\right)\left(\frac{V_T}{\mathrm{Vol}\left(\mathcal{S}^*\right)}-1\right)
\right]
\nonumber\\
&=\epsilon,
\end{align}
where for any $\mathcal{S}\in\mathbb{S}_K$, $\rho_{\mathcal{S}}$ denotes the probability density function associated to $\mathbb{P}_{\mathcal{S}}$. This completes the proof.
\end{proof}
Next, we take advantage of a well-known generalization bound for consistent binary classifiers that has been originally implied from a series of theorems in  \cite{vapnik1982estimation}. We use a modified version of this bound which has been formulated in Theorem A2.1 of \cite{blumer1989learnability} and, for example, can also be found in Lemma 4 of \cite{hanneke2016optimal}. Formal statement of the above-mentioned generalization bound is rewritten in Lemma \ref{lemma:optimGeneral} for the sake of the readability of the paper. Before proceeding to the lemma, let us define the ``growth function" of a set of binary classifiers:
\begin{definition}
\label{def:growthfunction}
Assume $\mathcal{X}$ is an arbitrary set of feature vectors and $\mathscr{C}\subseteq\left\{0,1\right\}^{\mathcal{X}}$ denotes a corresponding set of binary classifiers. The growth function of $\mathscr{C}$, denoted by $\Pi_{\mathscr{C}}\left(\cdot\right):\mathbb{N}\rightarrow\mathbb{N}$, is defined such that for all $m\in\mathbb{N}$, $\Pi_{\mathscr{C}}\left(m\right)$ represents the maximum number of distinct ways that any set of $m$ points in $\mathcal{X}$ can be labeled by classifiers in $\mathscr{C}$. Mathematically speaking, we have
$$
\Pi_{\mathscr{C}}\left(m\right)\triangleq
\sup_{\boldsymbol{X}_1,\ldots,\boldsymbol{X}_m\in\mathcal{X}}~\left\vert\left\{
\left(c\left(\boldsymbol{X}_1\right),\ldots,c\left(\boldsymbol{X}_m\right)\right)\bigg\vert~c\in\mathscr{C}
\right\}\right\vert,
$$
where $\left\vert\left\{\cdot\right\}\right\vert$ denotes the cardinality of a set.
\end{definition}
\begin{lemma}
\label{lemma:optimGeneral}
Assume an arbitrary and unknown distribution $\mathbb{P}$ supported over a set $\mathcal{X}$, and a learnable class of binary classifiers $\mathscr{C}\subseteq{\left\{0,1\right\}}^{\mathcal{X}}$. Let $\Pi_{\mathscr{C}}:\mathbb{N}\rightarrow\mathbb{N}$ be the ``growth function" of $\mathscr{C}$. Assume a dataset $\tilde{\boldsymbol{D}}$ consisting of $n$ i.i.d. samples drawn from $\mathbb{P}$, where each sample has been labeled by a classifier $c_T\in\mathscr{C}$. 

Assume there exists algorithm $\mathscr{A}$ which given $\tilde{\boldsymbol{D}}$ outputs a consistent classifier $c^*\in\mathscr{C}$. Then, for any $\zeta>0$, the following generalization bound holds with probability at least $1-\zeta$:
$$
\mathbb{P}\left\{c^*\left(\boldsymbol{X}\right)\neq c_T\left(\boldsymbol{X}\right)\right\}\leq
O\left(\frac{\log\Pi_{\mathscr{C}}\left(2n\right) + \log\frac{1}{\zeta}}{n}\right).
$$
\end{lemma}
Proof can be found in any of the above-mentioned references. Moreover, Lemma \ref{lemma:GrowthFunc} proves that for the particular class of $K$-dimensional simplices, the growth function associated to the classifier set $\mathscr{C}$ can be bounded as
$\log\Pi_{\mathscr{C}}\left(n\right)\leq O\left(K^2\log\left(n/K\right)\right)$.
\begin{lemma}
For $K\in\mathbb{N}$, assume the classifier set $\mathscr{C}$ which consists of all the simplex-based binary classifiers in $\mathbb{R}^K$ as described in \eqref{eq:defmathscrC}. Then, the following upper-bound holds for the growth function $\Pi_{\mathscr{C}}$, for $n\in\mathbb{N}$:
\begin{equation*}
\Pi_{\mathscr{C}}\left(n\right)\leq
\left(\frac{ne}{K+1}\right)^{\left(K+1\right)^2}.
\end{equation*}
\label{lemma:GrowthFunc}
\end{lemma}
See Appendix \ref{app:aux} for the proof. Substituting into the generalization bound of Lemma \ref{lemma:optimGeneral}, we have the following bound for the total-variation error of the output of algorithm $\mathscr{A}$, denoted by $\mathcal{S}^*$:
\begin{equation}
\mathcal{D}_{\mathrm{TV}}\left(\mathbb{P}_{\mathcal{S}_T},\mathbb{P}_{\mathcal{S}^*}\right)\leq
O\left(\frac{K^2\log\frac{n}{K} + \log\frac{1}{\zeta}}{n}\right).
\end{equation}
Following some simple algebra, one can show that the minimum number of samples $n$ which guarantees a maximum total-variation error of $\epsilon>0$ with probability at least $1-\zeta$ for algorithm $\mathscr{A}$ must satisfy the following sample complexity bound:
\begin{equation}
n\geq O\left(\frac{K^2\log\frac{K}{\epsilon}+\log\frac{1}{\zeta}}{\epsilon}\right).
\end{equation}
What remains to prove is that the ML algorithm has the same {\it {consistency}} and {\it {bounded volume property}} as the imaginary algorithm $\mathscr{A}$. First, remember that all the data samples in the dataset $\tilde{\boldsymbol{D}}$ are labeled with $0$. On the other hand, MLE always finds a simplex that encapsulates all the data points $\boldsymbol{X}_1,\ldots,\boldsymbol{X}_n$. Therefore, the classifier in $\mathscr{C}$ which corresponds to the MLE of $\mathcal{S}_T$, say $c\left(\cdot;\mathcal{S}^*_{\mathrm{ML}}\right)$, is always a consistent classifier. This proves the consistency property. In addition, ML always finds the smallest (in terms of Lebesgue measure or volume) simplex that encapsulates all the samples. It should be noted that the true simplex $\mathcal{S}_T$ also corresponds to a consistent classifier in $\mathscr{C}$, which means the volume of $\mathcal{S}^*_{\mathrm{ML}}$ must be smaller or equal to the volume of $\mathcal{S}_T$ which is denoted by $V_T$, i.e.
\begin{equation}
\mathrm{Vol}\left(\mathcal{S}^*_{\mathrm{ML}}\right)\leq V_T.
\end{equation}
Therefore, ML algorithm satisfies the two properties mentioned above and the proof is complete.
\end{proof}


\begin{proof}[Proof of Theorem \ref{thm:PACnoiseless}]
Recall that $\mathcal{S}^*$ is obtained as
\begin{equation}
\mathcal{S}^*\triangleq~
\argmin_{\mathcal{S}\in\mathbb{S}_K}~\left\{
\hat{R}_{\mathrm{CRR}}\left(\mathcal{S}\right)
\triangleq
\frac{1}{n}\sum_{i=1}^{n}
\phi\left(\boldsymbol{X}_i,\mathcal{S}\right) + \gamma\mathrm{Vol}\left(\mathcal{S}\right)
\right\},
\end{equation}
where for $\boldsymbol{X}\in\mathbb{R}^K$,
$\phi\left(\boldsymbol{X},\mathcal{S}\right)\triangleq
\sqrt{n}\ell\left(d_{\mathcal{S}}\left(\boldsymbol{X}\right)\right)$. Instead of working directly with $\mathcal{S}^*$ and obtaining an upper bound on $\mathcal{D}_{\mathrm{TV}}\left(\mathbb{P}_{\mathcal{S}^*},\mathbb{P}_{\mathcal{S}_T}\right)$, we define the following set
\begin{align}
\mathcal{Q}\left(\epsilon \right)\triangleq
&\left\{
\mathcal{S}\in\mathbb{S}_K \bigg\vert~
\mathrm{Vol}\left(\mathcal{S}\right)\leq V_T~,
\right.
\\
&\hspace{19mm}
\mathrm{Vol}\left(\mathcal{S}\cap\mathcal{S}_T\right)\geq (1-\epsilon) V_T  ~,
\nonumber\\
&\hspace{19mm}
\left.
\max_{\boldsymbol{u}\in\mathcal{S}_T}d_{\mathcal{S}}\left(\boldsymbol{u}\right)\leq3\mathrm{Diam}\left(\mathcal{S}_T\right)
\right\},
\nonumber
\end{align}
where the diameter of a simplex $\mathcal{S}\in\mathbb{S}_K$, denoted by  $\mathrm{Diam}\left(\mathcal{S}\right)$, is defined as
\begin{equation}
\mathrm{Diam}\left(\mathcal{S}\right)\triangleq
\max_{\boldsymbol{u},\boldsymbol{u}'\in\mathcal{S}}\left\Vert
\boldsymbol{u}-\boldsymbol{u}'
\right\Vert_2.
\end{equation}
Then, we prove the following two claims:
\\
\textbf{Claim 1}: For all $\mathcal{S}\in \mathcal{Q}\left(\epsilon\right)$, we have $\mathcal{D}_{\mathrm{TV}}\left(\mathbb{P}_{\mathcal{S}},\mathbb{P}_{\mathcal{S}_T}\right)\leq \epsilon$. This is essentially a direct consequence of Lemma \ref{lemma:TVlemma} which has been stated and subsequently proved in Appendix \ref{appendix:Theorems}.
\\
\textbf{Claim 2}: The probability that $\mathcal{S}^* \in \mathcal{Q}\left(\epsilon\right)$ is at least $1-\zeta$, if the number of samples satisfies the following bound:
\begin{equation}
n\ge \left(
\frac{
6\ell\left(3\underline{\lambda}KV^{\frac{1}{K}}_T\right)
\left(
\sqrt{K^2\log\frac{ne}{K}}
+
\sqrt{\log\frac{1}{\zeta}}
\right) + \gamma V_T\epsilon}
{L\left(\frac{\epsilon V^{1/K}_T}{\left(K+1\right)\bar{\lambda}}\right)\epsilon}
\right)^2.
\end{equation}
In fact, combining the two claims will prove the theorem. Next, we prove each claim separately.

\begin{proof}[Proof of Claim 2]
Claim 2 can be shown by proving that the following  three conditions are held with probability at least $1-\zeta$:
\\
\textbf{C1}: $V^*\triangleq \mathrm{Vol}\left(\mathcal{S}^*\right)\leq V_T$, and consequently, the volume of the solution to \eqref{eq:proposed} is bounded.
\\
\textbf{C2}: $\max_{\boldsymbol{u}\in\mathcal{S}_T}d_{\mathcal{S}^*}\left(\boldsymbol{u}\right)\leq3\mathrm{Diam}\left(\mathcal{S}_T\right)$, which means $\mathcal{S}^*$ cannot be very far from $\mathcal{S}_T$.
\\
\textbf{C3}: $\mathrm{Vol}\left(\mathcal{S}^*\cap\mathcal{S}_T\right)\geq \left(1-\epsilon\right) V_T$.
\\
The first two conditions \textbf{C1} and \textbf{C2} do hold with certainty as the following lemma presents.

\begin{lemma}
With certainty, we have $V^* \leq V_T$ and $max_{\boldsymbol{u}\in\mathcal{S}_T}d_{\mathcal{S}^*}\left(\boldsymbol{u}\right)\leq3\mathrm{Diam}\left(\mathcal{S}_T\right)$. 
\label{lemma:solSpace}
\end{lemma}
See Appendix \ref{app:aux} for the proof. From Lemma \ref{lemma:solSpace}, one can deduce that $\mathcal{S}^*$  belongs to either $\mathcal{Q}\left(\epsilon\right)$ or its partial complement $\mathcal{Q}^c\left(\epsilon\right)$ which is defined as
\begin{align}
\mathcal{Q}^c\left(\epsilon \right)\triangleq
&\left\{
\mathcal{S}\in\mathbb{S}_K \bigg\vert~
\mathrm{Vol}\left(\mathcal{S}\right)\leq V_T~,
\right.
\\
&\hspace{19mm}
\mathrm{Vol}\left(\mathcal{S}\cap\mathcal{S}_T\right)< (1-\epsilon) V_T  ~,
\nonumber\\
&\hspace{19mm}
\left.
\max_{\boldsymbol{u}\in\mathcal{S}_T}d_{\mathcal{S}}\left(\boldsymbol{u}\right)\leq3\mathrm{Diam}\left(\mathcal{S}_T\right)
\right\},
\nonumber
\end{align}
For a given $\epsilon$, we are interested in finding $n$ such that 
$\mathcal{S}^*\in\mathcal{Q}(\epsilon)$ with probability at least $1-\zeta$. To this end, we seek an appropriate function $u(n,\epsilon)$ such that 
\begin{equation}
\mathbb{P}\left\{ \inf_{\mathcal{S}\in \mathcal{Q}^c (\epsilon)} \hat{R}_{\mathrm{CRR}}\left(\mathcal{S}\right) - \hat{R}_{\mathrm{CRR}}\left(\mathcal{S}_T\right) > u(n,\epsilon) \right\} > 1-\zeta.
\end{equation}
Next, we find $n$ such that $u(n,\epsilon)$ is non-negative. In this way, one can deduce that $\mathcal{S}_T$ is preferable over all the simplices in $\mathcal{Q}^c(\epsilon)$ with probability at least $1-\zeta$. Since $\mathcal{S}^*$ is the optimal solution, it is in $\mathcal{Q}(\epsilon)$ with probability at least $1-\zeta$. To this aim, we obtain a candidate for $u(n,\epsilon)$ by considering two facts for each $\mathcal{S}\in\mathcal{Q}^c\left(\epsilon\right)$. First,
$\hat{R}_{\mathrm{CRR}}\left(\mathcal{S}\right)$ is sufficiently close to its statistical average $
\mathbb{E}_{\mathcal{S}_T}
\hat{R}_{\mathrm{CRR}}\left(\mathcal{S}\right)$. Second, $
\mathbb{E}_{\mathcal{S}_T}
\hat{R}_{\mathrm{CRR}}\left(\mathcal{S}\right)$ is sufficiently larger than $\hat{R}_{\mathrm{CRR}}\left(\mathcal{S}_T\right)$. The main challenge is to show that these two facts hold for all simplices in $\mathcal{Q}^c(\epsilon)$, with high probability. To prove the first fact, we define $\phi_{\max}$ as
\begin{equation}
\phi_{\max}\triangleq \sup_{\boldsymbol{X}\in\mathcal{S}_T}~\sup_{\mathcal{S}\in\mathcal{Q}(\epsilon)\cup \mathcal{Q}^c(\epsilon) }\phi\left(\boldsymbol{X},\mathcal{S}\right),
\label{eq:gmax}
\end{equation}
and use it to normalize the function $\phi$ as
\begin{equation}
\bar{\phi}\left(\boldsymbol{X},\mathcal{S}\right)\triangleq \frac{
\sqrt{n}\ell\left(d_{\mathcal{S}}\left(\boldsymbol{X}\right)\right)
}{\phi_{\max}}.
\label{eq:gbardef}
\end{equation}
For any $\boldsymbol{X}\in\mathcal{S}_T$ and $\mathcal{S}\in\mathcal{Q}(\epsilon)\cup \mathcal{Q}^c(\epsilon)$, we have:
\begin{align}
\phi_{\max}=
\sup_{\boldsymbol{X},\mathcal{S}}~
\sqrt{n}\ell\left(
d_{\mathcal{S}}\left(\boldsymbol{X}\right)
\right)
\leq
\sqrt{n}\ell\left(
3\max_{\boldsymbol{u},\boldsymbol{u}'\in\mathcal{S}_T}~
\left\Vert
\boldsymbol{u}-\boldsymbol{u}'
\right\Vert_2
\right) 
\leq \sqrt{n}\ell\left(3\underline{\lambda}KV^{\frac{1}{K}}_T\right),
\end{align}
which holds due to both the $\left(\underline{\lambda},\bar{\lambda}\right)$-isoperimetricity assumption on $\mathcal{S}_T$, and $\ell$ being an increasing function. In more details, we already know from Definition \ref{def:shape} that 
\begin{equation}
\max_{\boldsymbol{u},\boldsymbol{u}'\in\mathcal{S}_T}~
\left\Vert
\boldsymbol{u}-\boldsymbol{u}'
\right\Vert_2
\leq
\max_{\boldsymbol{\theta},\boldsymbol{\theta}'\in\boldsymbol{\Theta_T}}~
\left\Vert
\boldsymbol{\theta}-\boldsymbol{\theta}'
\right\Vert_2
\leq
\underline{\lambda}KV^{\frac{1}{K}}_T.
\end{equation}
Also, since $\ell$ is increasing, the maximum of its output occurs when the input argument is maximized. Defining the function set $\Phi$ as
\begin{equation*}
\Phi\triangleq
\left\{
\bar{\phi}\left(\cdot,\mathcal{S}\right):\mathcal{S}_T\rightarrow\left[0,1\right]
\big\vert~\mathcal{S}\in\mathcal{Q}^c\left(\epsilon\right)
\right\},
\end{equation*}
and using one of the fundamental theorems in Vapnik-Chervonenkis (VC) theory of statistical learning gives us the following bound, for all $\zeta>0$:
\begin{equation}
\mathbb{P}\left\{
\sup_{\mathcal{S}\in\mathcal{Q}^c(\epsilon)}
\left\vert
\frac{
\hat{R}_{\mathrm{CRR}}\left(\mathcal{S}\right)
-
\mathbb{E}_{\mathcal{S}_T}
\hat{R}_{\mathrm{CRR}}\left(\mathcal{S}\right)
}{\phi_{\max}}
\right\vert \stackrel{\mathrm{(a)}}{\leq}
2\mathfrak{R}_n\left\{\Phi\right\} + \sqrt{\frac{\log\left({2}/{\zeta}\right)}{2n}}
\right\}\stackrel{\mathrm{(b)}}{\ge} 1-\zeta,
\label{eq:VCbound}
\end{equation}
where $\mathfrak{R}_n\left\{\Phi\right\}$ represents the {\it {Rademacher complexity}} of the function set $\Phi$, for sample size $n$ and dataset $\boldsymbol{D}=\left\{\boldsymbol{X}_1,\ldots,\boldsymbol{X}_n\right\}\sim\mathbb{P}_{\mathcal{S}_T}$ \cite{mohri2012foundations}. 

One of the main challenges is to find an appropriate upper bound on  $\mathfrak{R}_n\left\{\Phi\right\}$. In the following lemma, we obtain an upper bound showing that both terms in the r.h.s. of inequality (a) in \eqref{eq:VCbound} are $\tilde{O}\left(n^{-1/2}\right)$.

\begin{lemma}
\label{lemma2}
Rademacher complexity of the function set $\Phi$, for all $\epsilon>0$ and increasing functions $\ell:\mathbb{R}\rightarrow\mathbb{R}$ can be bounded as
\begin{equation*}
\mathfrak{R}_n\left\{\Phi\right\}\leq
\sqrt{\frac{4K^2}{n}\log\left(\frac{ne}{K}\right)},
\end{equation*}
for $n,K\in\mathbb{N}$
\end{lemma}
See Appendix \ref{app:aux} for the proof. The second fact, i.e. for all $\mathcal{S}\in\mathcal{Q}^c\left(\epsilon\right)$, the value of $\mathbb{E}_{\mathcal{S}_T}\hat{R}_{\mathrm{CRR}}\left(\mathcal{S}\right)$ is sufficiently larger than $\hat{R}_{\mathrm{CRR}}\left(\mathcal{S}_T\right)$, is formalized by the following lemma.

\begin{lemma}
Let $\mathcal{S}_T\in\mathbb{S}_K$  be $\left(\underline{\lambda},\bar{\lambda}\right)$-isoperimetric and assume $\boldsymbol{D}\triangleq\left\{\boldsymbol{X}_1,\ldots,\boldsymbol{X}_n\right\}$ are $n$ i.i.d. samples drawn uniformly from $\mathcal{S}_T$. For any increasing and integrable function $\ell$, all $\gamma\ge 0$ and $0< \epsilon\leq 1$, 
assume we have
\begin{equation}\label{eq:condition_on_n}
n\ge
\left(
\frac{\gamma V_T}{\ell\left(\frac{\epsilon V^{1/K}_T}{\left(K+1\right)\bar{\lambda}}\right)}
\right)^2.
\end{equation}
Then, the following bound holds:
\begin{align}
\inf_{\mathcal{S}\in\mathcal{Q}^c\left(\epsilon\right)}\mathbb{E}_{\mathcal{S}_T}\left\{
\hat{R}_{\mathrm{CRR}}\left(\mathcal{S};\boldsymbol{D},\ell,\gamma\right)
\right\}\ge \left(
\sqrt{n}L\left(\frac{\epsilon V^{1/K}_T}{\left(K+1\right)\bar{\lambda}}\right) - 
\gamma V_T
\right)\epsilon
+ \hat{R}_{\mathrm{CRR}}\left(\mathcal{S}_T;\boldsymbol{D},\ell,\gamma\right),
\end{align}
where $L\left(x\right)\triangleq\int_{0}^{x}\ell\left(u\right)\mathrm{d}u - \ell\left(0\right)$.
\label{lemma:ExpLossBound}
\end{lemma}
See Appendix \ref{app:aux} for the proof. From  \eqref{eq:VCbound} and Lemma \ref{lemma2},  the following lower-bound holds for all $\mathcal{S}\in\mathcal{Q}^c\left(\epsilon\right)$ with probability at least $1-\zeta$:
\begin{equation}
\hat{R}_{\mathrm{CRR}}\left(\mathcal{S}\right)\ge
\mathbb{E}_{\mathcal{S}_T}
\hat{R}_{\mathrm{CRR}}\left(\mathcal{S}\right)
- 
\ell\left(3\underline{\lambda}KV^{\frac{1}{K}}_T\right)
\left(
4\sqrt{K^2\log\frac{ne}{K}}
+
\sqrt{\log\frac{2}{\zeta}}
\right).
\end{equation}
Substituting from Lemma \ref{lemma:ExpLossBound}, for all $\mathcal{S}\in\mathcal{Q}^c\left(\epsilon\right)$ and with probability at least $1-\zeta$, we have
\begin{align}
\hat{R}_{\mathrm{CRR}}\left(\mathcal{S}\right)-
\hat{R}_{\mathrm{CRR}}\left(\mathcal{S}_T\right)
\ge~&
\left(
\sqrt{n}L\left(\frac{\epsilon V^{1/K}_T}{\left(K+1\right)\bar{\lambda}}\right) - 
\gamma V_T
\right)\epsilon-
\nonumber
\\
&
\ell\left(3\underline{\lambda}KV^{\frac{1}{K}}_T\right)
\left(
4\sqrt{K^2\log\frac{ne}{K}}
+
\sqrt{\log\frac{2}{\zeta}}
\right)\triangleq u(n,\epsilon),
\label{eq:sampCompThm1}
\end{align}
subject to $n$ satisfies \eqref{eq:condition_on_n}. In this way, we have obtained the desired $u(n,\epsilon)$. One can show that for
\begin{equation}
n\ge \left(
\frac{
6\ell\left(3\underline{\lambda}KV^{\frac{1}{K}}_T\right)
\left(
\sqrt{K^2\log\frac{ne}{K}}
+
\sqrt{\log\frac{1}{\zeta}}
\right) + \gamma V_T\epsilon}
{L\left(\frac{\epsilon V^{1/K}_T}{\left(K+1\right)\bar{\lambda}}\right)\epsilon}
\right)^2,
\end{equation}
both  \eqref{eq:condition_on_n} and $u\left(n,\epsilon\right)>0$ are satisfied. Therefore,
\begin{equation}
\mathbb{P}\left\{ \inf_{\mathcal{S}\in \mathcal{Q}^c\left(\epsilon\right)} \hat{R}_{\mathrm{CRR}}\left(\mathcal{S}\right) - \hat{R}_{\mathrm{CRR}}\left(\mathcal{S}_T\right) > 0 \right\} > 1-\zeta,
\end{equation}
which implies
$\mathbb{P}\left\{ \mathcal{S}^*
\in \mathcal{Q}\left(\epsilon\right)
\right\} > 1-\zeta$
and completes the proof of Claim 2.
\end{proof}
As mentioned earlier, combining Claims 1 and 2 proves Theorem \ref{thm:PACnoiseless}.
\end{proof}
\begin{proof}[Proof of Theorem \ref{lemma:Grad}]
According to the definition of $d_{\mathcal{S}}\left(\cdot\right)$, having $d_{\mathcal{S}}\left(\boldsymbol{X}\right)=0$ immediately implies that
$\boldsymbol{w}^T_k\boldsymbol{X}+b_k\leq0,~k=0,\ldots,K$. However, we have already assumed that $\boldsymbol{w}^T_k\boldsymbol{X}+b_k<0$ for all $k=0,\ldots,K$, which means $\boldsymbol{X}$ is inside $\mathcal{S}$ with a positive margin from its boundary. Therefore, an infinitesimally small perturbation of $\boldsymbol{\Theta}$ (regardless of the direction) cannot throw $\boldsymbol{X}$ out of $\mathcal{S}\left(\boldsymbol{\Theta}\right)$ and thus the gradient of $d_{\mathcal{S}\left(\boldsymbol{\Theta}\right)}\left(\boldsymbol{X}\right)$ is zero.

For the case of $d_{\mathcal{S}\left(\boldsymbol{\Theta}\right)}\left(\boldsymbol{X}\right)>0$, one can deduce that $\boldsymbol{w}^T_k\boldsymbol{X}+b_k>0$ for at least one $k\in\left\{0,\ldots,K\right\}$. Due to the assumption of lemma, $k^*\triangleq\argmax_{k}~\boldsymbol{w}^T_k\boldsymbol{X}+b_k$ is unique. This alternatively implies that there exists $\delta>0$ such that for all $\boldsymbol{u}\in\mathbb{R}^K$ with $\left\Vert{\boldsymbol{u}}-\boldsymbol{X}\right\Vert_2\leq \delta$ we have $d_{\mathcal{S}\left(\boldsymbol{\Theta}\right)}\left({\boldsymbol{u}}\right)=\boldsymbol{w}^T_{k^*}{\boldsymbol{u}}+b_{k^*}$. Therefore, we have
\begin{equation}
\nabla_{\boldsymbol{\Theta}}d_{\mathcal{S}\left(\boldsymbol{\Theta}\right)}\left(\boldsymbol{X}\right)
=
\nabla_{\boldsymbol{\Theta}}
\left(
\min_{\boldsymbol{p}\in\mathbb{R}^K\vert~\boldsymbol{p}^T\boldsymbol{1}=1}~
f\left(\boldsymbol{X},\boldsymbol{p}\right)
\triangleq
\left\Vert
\boldsymbol{X}-\boldsymbol{\Theta}_{-k^*}\boldsymbol{p}
\right\Vert_2
\right),
\end{equation}
which is based on the fact that $d_{\mathcal{S}\left(\boldsymbol{\Theta}\right)}\left(\boldsymbol{X}\right)$ is the minimum Euclidean distance of $\boldsymbol{X}$ from $\mathcal{H}_{k^*}=\left\{\boldsymbol{\Theta}_{-k^*}\boldsymbol{p}\vert~\boldsymbol{p}\in\mathbb{R}^K~,~\boldsymbol{p}^T\boldsymbol{1}=1\right\}$ which represents the hyper-plane including the $k^*$th facet of $\mathcal{S}$. Due to the continuity of $f$, one can use the envelope theorem and obtain the following formulation for the gradient:
\begin{equation}
\nabla_{\boldsymbol{\Theta}_{-k^*}}d_{\mathcal{S}\left(\boldsymbol{\Theta}\right)}\left(\boldsymbol{X}\right)
=
\nabla_{\boldsymbol{\Theta}_{-k^*}}\left\Vert
\boldsymbol{X}-\boldsymbol{\Theta}_{-k^*}\boldsymbol{p}^*
\right\Vert_2
=
\frac{\left(\boldsymbol{\Theta}_{-k^*}\boldsymbol{p}^*-\boldsymbol{X}\right)\boldsymbol{p}^{*T}}
{\left\Vert\boldsymbol{\Theta}_{-k^*}\boldsymbol{p}^*-\boldsymbol{X}\right\Vert_2},
\label{eq:gradSpareEq}
\end{equation}
where
\begin{equation}
\boldsymbol{p}^*\triangleq\argmin_{\boldsymbol{p}\in\mathbb{R}^K\vert~\boldsymbol{p}^T\boldsymbol{1}=1}~
\left\Vert
\boldsymbol{X}-\boldsymbol{\Theta}_{-k^*}\boldsymbol{p}
\right\Vert_2.
\end{equation}
It should not be surprising that derivatives w.r.t. $\boldsymbol{\theta}_{k^*}$ is zero since the $k^*$th vertex does not contribute to the planar distance from $\mathcal{H}_{k^*}$. Also, note that $\boldsymbol{p^*}$ can be computed analytically. More precisely, $\boldsymbol{\Theta}_{-k^*}\boldsymbol{p}^*$ is the projection of $\boldsymbol{X}$ into the $\left(K-1\right)$-dimensional subspace $\mathcal{H}_{k^*}$. Thus, we have
\begin{align}
\boldsymbol{\Theta}_{-k^*}\boldsymbol{p}^*&=\boldsymbol{X}-\left(
\boldsymbol{w}^T_{k^*}\boldsymbol{X}+b_{k^*}
\right)\boldsymbol{w}_{k^*},
\\
\mathrm{with}\quad b_{k^*}&=\frac{-1}{K}\boldsymbol{w}^T_{k^*}\sum_{k\neq k^*}\boldsymbol{\theta}_{k}=-\frac{1}{K}\boldsymbol{\Theta}_{-k^*}\boldsymbol{1},
\nonumber
\end{align}
where $\boldsymbol{w}_{k^*}$, as mentioned in the lemma, is the normal outward vector of the $k^*$th facet of $\mathcal{S}$. In order to compute $\boldsymbol{w}_{k^*}$, it should be noted that $\boldsymbol{w}_{k^*}$ must be orthogonal to all $\boldsymbol{\theta}_k-\frac{1}{K}\boldsymbol{\Theta}_{-k^*}\boldsymbol{1}$ for $k\neq k^*$. Thus, we have 
\begin{equation}
\left(\boldsymbol{\Theta}_{-k^*}-\frac{1}{K}
\boldsymbol{\Theta}_{-k^*}\boldsymbol{1}\boldsymbol{1}^T
\right)^T\boldsymbol{w}_{k^*}=
\left(\boldsymbol{I}-\frac{\boldsymbol{1}\boldsymbol{1}^T}{K}\right)
\boldsymbol{\Theta}^T_{-k^*}\boldsymbol{w}_{k^*}
=0.
\end{equation}
Since the columns of $\boldsymbol{\Theta}_{-k^*}-\frac{1}{K}
\boldsymbol{\Theta}_{-k^*}\boldsymbol{1}\boldsymbol{1}^T$, or equivalently
$\boldsymbol{\theta}_k-\frac{1}{K}\boldsymbol{\Theta}_{-k^*}\boldsymbol{1}$ for $k\neq k^*$, all lie on the hyper-plane $\mathcal{H}_{k^*}$, then $\left(\boldsymbol{I}-\frac{1}{K}\boldsymbol{1}\boldsymbol{1}^T\right)
\boldsymbol{\Theta}^T_{-k^*}$ must have at least one zero eigenvalue. On the other hand, since we have already assumed that $\mathcal{S}\left(\boldsymbol{\Theta}\right)$ has a non-zero volume, the projections of the mentioned columns on $\mathcal{H}_{k^*}$ are linearly independent in $\mathbb{R}^{K-1}$, and therefore the above-mentioned zero eigenvalue is unique. Thus, $\boldsymbol{w}_{k^*}$ can be obtained as described in the statement of lemma. In this regard, we have
\begin{align}
\boldsymbol{\Theta}_{-k^*}\boldsymbol{p}^*=\boldsymbol{X}-\left(
\boldsymbol{w}^T_{k^*}\boldsymbol{X}+b_{k^*}
\right)\boldsymbol{w}_{k^*}
=
\boldsymbol{X}-\boldsymbol{w}_{k^*}\boldsymbol{w}^T_{k^*}
\left(
\boldsymbol{X}-\frac{1}{K}\boldsymbol{\Theta}_{-k^*}\boldsymbol{1}
\right),
\end{align}
which means $\boldsymbol{p}^*\triangleq\boldsymbol{\Theta}^{\dagger}_{-k^*}\left(\boldsymbol{X}-\boldsymbol{w}_{k^*}\boldsymbol{w}_{k^*}^T\left(\boldsymbol{X}-\frac{1}{K}\boldsymbol{\Theta}_{-k^*}\boldsymbol{1}\right)\right)\in\mathbb{R}^K$. Note that using the pseudo-inverse operator does not result into any information loss here since $\boldsymbol{X}-\left(
\boldsymbol{w}^T_{k^*}\boldsymbol{X}+b_{k^*}
\right)\boldsymbol{w}_{k^*}$ lies on the hyper-plane specified by $\boldsymbol{\Theta}_{-k^*}$. Finally, we have
$\nabla_{\boldsymbol{\Theta}_{-k^*}}\left(\cdot\right)=\nabla_{\boldsymbol{\Theta}}\left(\cdot\right)\boldsymbol{J}_{k^*}$, and $\nabla_{\boldsymbol{\Theta}}\ell\left(\cdot\right)=\ell'\left(\cdot\right)\nabla_{\boldsymbol{\Theta}}\left(\cdot\right)$. Considering the fact that
\begin{equation}
\boldsymbol{w}_{k^*}=-\frac{\boldsymbol{\Theta}_{-k^*}\boldsymbol{p}^*-\boldsymbol{X}}
{\left\Vert\boldsymbol{\Theta}_{-k^*}\boldsymbol{p}^*-\boldsymbol{X}\right\Vert_2},
\end{equation}
and substituting into \eqref{eq:gradSpareEq}, completes the proof.
\end{proof}

\section{Proofs of Lemmas and Corollaries}
\label{app:aux}

\begin{proof}[Proof of Corollary \ref{corl:softML}]
Proof is similar to that of Corollary \ref{corl:ML}. First, let us compute the integral function $L\left(\cdot\right)$:
\begin{equation}
L\left(x\right)=\frac{1}{x}\int_{0}^{x}\left(1-e^{-bu}\right)\mathrm{d}u=
1-\left(\frac{1-e^{-bx}}{bx}\right),\quad
x\triangleq \frac{\epsilon V^{1/K}_T}{\left(K+1\right)\bar{\lambda}}.
\end{equation}
Considering the fact that $b=K/\epsilon$, it can be easily verified that $L\left(x\right)$ has a positive lower-bound which is independent of $K$, $\epsilon$ and $\zeta$. On the other hand, we have $\ell\left(\cdot\right)\leq 1$ due to the particular choice of $\ell$. Substituting the above results into the sample complexity of Theorem \ref{thm:PACnoiseless}, and using the same techniques which we already used for the proof Theorem \ref{corl:ML} will complete the proof.
\end{proof}


\begin{proof}[Proof of Lemma \ref{lemma:GrowthFunc}]
The proof is motivated by the fact that any $K$-simplex  $\mathcal{S}\in\mathbb{S}_K$ can be viewed as a set of points in $\mathbb{R}^{K}$ satisfying the following set of constraints:
\begin{equation}
\mathcal{S}=\left\{\boldsymbol{u}\in\mathbb{R}^K\big\vert~
\boldsymbol{w}^T_k\boldsymbol{u}+b_k\leq0,~\forall k=0,1,\ldots,K
\right\},
\end{equation}
where for each $k$, the couple $H_k\triangleq\left(\boldsymbol{w}_k,b_k\right)\in\mathbb{R}^K\times\mathbb{R}$ represent a $\left(K-1\right)$-dimensional linear hyper-plane that includes the $k$th facet of $\mathcal{S}$. Not to mention that $\boldsymbol{w}_k$ are considered to be outwards w.r.t. $\mathcal{S}$. We can then write $\mathcal{S}=\mathcal{S}\left(H_0,\ldots,H_K\right)$.

Assume one aims to use a $K$-simplex as a binary classifier, where for any $\mathcal{S}\in\mathbb{S}_K$ and $\boldsymbol{x}\in\mathbb{R}^K$, $\boldsymbol{x}$ is labeled as $+1$ if $\boldsymbol{x}\in\mathcal{S}$ and is labeled $-1$, otherwise. Then, we are interested to bound the maximum number of ways that a set of $n$ generally-positioned points in $\mathbb{R}^{K}$ can be labeled via simplices in $\mathbb{S}_K$, i.e. the growth function of $\mathscr{C}$ or $\Pi_{\mathscr{C}}\left(n\right)$. Fortunately, this goal can be achieved by using the growth function of hyper-planes in $\mathbb{R}^K$, which is already derived.

Let us define $\mathcal{L}_{K+1}\triangleq\left\{-1,+1\right\}^{K+1}$ as the set of all binary {\it {super labels}} of dimension $K+1$ and the function family $\mathcal{H}$, such that $h=h\left(\boldsymbol{w}_{0:K},b_{0:K}\right)\in\mathcal{H}:\mathbb{R}^{K}\rightarrow\mathcal{L}_{K+1}$ is defined as:
\begin{equation}
h\left(\boldsymbol{x}\right)\triangleq\left[
\mathrm{sign}\left(\boldsymbol{w}_0^T\boldsymbol{x}+b_0\right)~,~\ldots~,~
\mathrm{sign}\left(\boldsymbol{w}_{K}^T\boldsymbol{x}+b_{K}\right)
\right]^T,~\forall \boldsymbol{x}\in\mathbb{R}^{K},
\end{equation}
where we define $\mathrm{sign}\left(0\right)=-1$. Then, by letting $\mathcal{I}_{K}$ to be the set of all hyper-plane binary classifiers in $\mathbb{R}^{K}$ and according to the {\it {rule of product}}, the growth function of $\mathcal{H}$ can be computed as
\begin{equation}
\Pi_{\mathcal{H}}\left(n\right)=\underbrace{\Pi_{\mathcal{I}_{K}}\left(n\right)\times \cdots\times \Pi_{\mathcal{I}_{K}}\left(n\right)}_{K+1 ~\mathrm{times}}.
\end{equation}
It is easy to see that $\Pi_{\mathscr{C}}\left(n\right)\leq\Pi_\mathcal{H}\left(n\right)$. In fact, for any $\mathcal{S}=\mathcal{S}\left(H_0,\ldots,H_K\right)\in\mathbb{S}_K$, the hyper-plane parameter sets $H_0,\ldots,H_K$ uniquely identify a function in $\mathcal{H}$ that can assign super labels to $\mathbb{R}^K$, while only a subset of them can be generated by $\mathcal{S}$. On the other hand, considering the fact that VC dimension of hyper-planes in $\mathbb{R}^{K}$ is exactly $K+1$, and using the Sauer-Shelah lemma (see, for example, Theorem 3.5 of \cite{mohri2012foundations}), we have the following bound for $\Pi_{\mathcal{I}_{K}}\left(n\right)$ :
\begin{equation}
\Pi_{\mathcal{I}_{K}}\left(n\right)\leq
\sum_{i=0}^{K+1}\binom{n}{i}\leq\left(\frac{ne}{K+1}\right)^{K+1},
\end{equation}
which, according to the previous inequality, means: $\Pi_{\mathscr{C}}\left(n\right)\leq\left(\frac{ne}{K+1}\right)^{\left(K+1\right)^2}$ and completes the proof.
\end{proof}


\begin{proof}[Proof of Lemma \ref{lemma:solSpace}]
First, we prove $V^*\triangleq\mathrm{Vol}\left(\mathcal{S}^*\right)\leq V_T$; The proof is by contradiction. Assume $V^*>V_T$. Since $d_{\mathcal{S}_T}\left(\boldsymbol{X}_i\right)=0$ for $i=1,\ldots,n$, and the fact that $\ell$ is an increasing function, it can be readily seen that
\begin{align}
\hat{R}_{\mathrm{CRR}}\left(\mathcal{S}^*\right)
\ge
\sqrt{n}\ell\left(0\right)+\gamma V^* > \sqrt{n}\ell\left(0\right)+\gamma V_T
=\hat{R}_{\mathrm{CRR}}\left(\mathcal{S}_T\right),
\end{align}
which contradicts the fact that $\mathcal{S}^*$ is the minimizer. 

Next, we show that 
\begin{equation}
\max_{\boldsymbol{u}\in\mathcal{S}_T}d_{\mathcal{S}^*}\left(\boldsymbol{u}\right)
\leq
3\max_{\boldsymbol{u},\boldsymbol{u}'\in\mathcal{S}_T}\left\Vert
\boldsymbol{u}-\boldsymbol{u}'
\right\Vert_2,
\label{eq:leq3Rad}
\end{equation}
where $\mathcal{S}^*$ is the minimizer of \eqref{eq:proposed} for the dataset $\boldsymbol{D}=\left\{\boldsymbol{X}_1,\ldots,\boldsymbol{X}_n\right\}$ that includes an arbitrary set of $n$ points $\boldsymbol{X}_i\in\mathcal{S}_T,~i\in\left[n\right]$. It should be noted that for this lemma, we do not need to assume any particular stochastic model for the generation of $\boldsymbol{X}_i$s, and these points can be arbitrarily chosen from $\mathcal{S}_T$. The main implication of the above inequality is the fact that $\mathcal{S}^*$ cannot be very far from the the points in $\boldsymbol{D}$, and the maximum planar distance between each $\boldsymbol{X}_i$ and $\mathcal{S}^*$ is bounded by a factor of the diameter of $\mathcal{S}_T$ ($\mathrm{Diam}\left(\mathcal{S}_T\right)\triangleq\max_{\boldsymbol{u},\boldsymbol{u}'\in\mathcal{S}_T}\left\Vert
\boldsymbol{u}-\boldsymbol{u}'
\right\Vert_2$).

Before attempting to prove \eqref{eq:leq3Rad}, we first show that for any $\mathcal{S}^*\in\mathbb{S}_K$ and $\boldsymbol{Z}\in\mathbb{R}^K$, the following useful bound holds for planar distance $d_{\mathcal{S}^*}\left(\boldsymbol{Z}\right)$:
\begin{equation}
d_{\mathcal{S}^*}\left(\boldsymbol{Z}\right)
\leq\min_{\boldsymbol{u}\in\mathcal{S}^*}\left\Vert
\boldsymbol{Z}-\boldsymbol{u}
\right\Vert_2.
\label{eq:sphVSpla}
\end{equation}
Let us denote the r.h.s. of \eqref{eq:sphVSpla} as the {\it {spherical distance}} of $\boldsymbol{Z}$ from $\mathcal{S}^*$. For $\boldsymbol{Z}\in\mathcal{S}^*$, both the planar distance and spherical distance of $\boldsymbol{Z}$ from $\mathcal{S}^*$ are zero and thus the inequality holds. For $\boldsymbol{Z}\notin\mathcal{S}^*$, first consider
$\boldsymbol{w}_k\in\mathbb{R}^K$ (with $\left\Vert\boldsymbol{w}_k\right\Vert_2=1$) and $b_k\in\mathbb{R}$, and then let $\mathcal{H}_k\triangleq\left\{\boldsymbol{x}\in\mathbb{R}^K\vert~\boldsymbol{w}^T_k\boldsymbol{x}+b_k=0\right\}$ represent the hyper-plane that encompasses the $k$th facet of $\mathcal{S}^*$. This way, the planar distance can be rewritten as
$d_{\mathcal{S}^*}\left(\boldsymbol{Z}\right)
=
\min_{\boldsymbol{u}\in\mathcal{H}_{k^*}}\left\Vert
\boldsymbol{Z}-\boldsymbol{u}
\right\Vert_2$, where $k^*\triangleq\argmax_{k}~\boldsymbol{w}^T_k\boldsymbol{Z}+b_k$. Also, for $\boldsymbol{Z}\notin\mathcal{S}^*$ we have $\boldsymbol{w}^T_{k^*}\boldsymbol{Z}+b_{k^*}>0$, while for all $\boldsymbol{u}\in\mathcal{S}^*$, we have $\boldsymbol{w}^T_{k^*}\boldsymbol{u}+b_{k^*}\leq0$, based on the fact that $\boldsymbol{w}_k$ vectors are defined as outward unitary vectors w.r.t. $\mathcal{S}^*$. Therefore, the line segment which connects $\boldsymbol{Z}$ to $\boldsymbol{u}$ must cross $\mathcal{H}_{k^*}$ at some point, which we denote by $\boldsymbol{Z}^*\left(\boldsymbol{u}\right)\in\mathcal{H}_{k^*}$. Note that we have:
\begin{equation}
\left\Vert\boldsymbol{Z}-\boldsymbol{u}\right\Vert_2=
\left\Vert\boldsymbol{Z}-\boldsymbol{Z}^*\left(\boldsymbol{u}\right)\right\Vert_2+
\left\Vert\boldsymbol{Z}^*\left(\boldsymbol{u}\right)-\boldsymbol{u}\right\Vert_2
\ge
\left\Vert\boldsymbol{Z}-\boldsymbol{Z}^*\left(\boldsymbol{u}\right)\right\Vert_2.
\end{equation}
Therefore, the following relations hold and \eqref{eq:sphVSpla} is proved:
\begin{align}
\min_{\boldsymbol{u}\in\mathcal{S}^*}\left\Vert\boldsymbol{Z}-\boldsymbol{u}\right\Vert_2
\ge
\min_{\boldsymbol{u}\in\mathcal{S}^*}\left\Vert\boldsymbol{Z}-\boldsymbol{Z}^*\left(\boldsymbol{u}\right)\right\Vert_2
\ge
\min_{\boldsymbol{u}\in\mathcal{H}_{k^*}}\left\Vert
\boldsymbol{Z}-\boldsymbol{u}
\right\Vert_2
=
d_{\mathcal{S}^*}\left(\boldsymbol{Z}\right).
\end{align}

Back to proving \eqref{eq:leq3Rad}, first let us define $\Bar{\boldsymbol{X}}\triangleq\frac{1}{n}\sum_{i}\boldsymbol{X}_i$ as the center of mass for data points in $\boldsymbol{D}$. Then, we show the following set of inequalities hold with certainty:
\begin{equation}
\boldsymbol{w}^T_k\Bar{\boldsymbol{X}}+b_k
\leq 2\mathrm{Diam}\left(\mathcal{S}_T\right),~k=0,\ldots,K.
\label{eq:claimRad}
\end{equation}
Based on the definition of diameter $\mathrm{Diam}\left(\cdot\right)$ for a simplex, the following inequality hold for any $\boldsymbol{Z}\in\mathcal{S}_T$:
\begin{equation}
\max_{\left\Vert\boldsymbol{w}\right\Vert_2=1}\boldsymbol{w}^T\left(
\boldsymbol{Z}-\Bar{\boldsymbol{X}}
\right)
\leq
\max_{\boldsymbol{u},\boldsymbol{u}'}\left\Vert\boldsymbol{u}-\boldsymbol{u}'\right\Vert_2=\mathrm{Diam}\left(\mathcal{S}_T\right),
\end{equation}
which also takes into account the fact that due to the convexity of a simplex, we have $\bar{\boldsymbol{X}}\in\mathcal{S}_T$. In this regard, we can prove \eqref{eq:claimRad} by contradiction: assume there exists $k$ such that $\boldsymbol{w}^T_k\Bar{\boldsymbol{X}}+b_k>2\mathrm{Diam}\left(\mathcal{S}_T\right)$. Then, one can show that for all $i\in\left[n\right]$ we have
\begin{align}
2\mathrm{Diam}\left(\mathcal{S}_T\right)
<
\boldsymbol{w}^T_k\Bar{\boldsymbol{X}}+b_k
=
\boldsymbol{w}^T_k\left(\Bar{\boldsymbol{X}}-\boldsymbol{X}_i\right)+
\boldsymbol{w}^T_k\boldsymbol{X}_i+b_k
\leq
\mathrm{Diam}\left(\mathcal{S}_T\right)+
\boldsymbol{w}^T_k\boldsymbol{X}_i+b_k,
\end{align}
or alternatively, $\boldsymbol{w}^T_k\boldsymbol{X}_i+b_k>\mathrm{Diam}\left(\mathcal{S}_T\right)$. This way, due to the definition of the planar distance in Section \ref{sec:notation}, we end up with the following lower-bound for $d_{\mathcal{S}^*}\left(\boldsymbol{X}_i\right)$ for all $i\in\left[n\right]$:
\begin{equation}
d_{\mathcal{S}^*}\left(\boldsymbol{X}_i\right)
=
\max\left\{
0~,~\max_{t}~\boldsymbol{w}^T_t\boldsymbol{X}_i+b_t
\right\}>\mathrm{Diam}\left(\mathcal{S}_T\right).
\end{equation}
The above result also gives a lower-bound for the continuously-relaxed risk in \eqref{eq:proposed} as $\hat{R}\left(\mathcal{S}^*;\boldsymbol{D}\right)>\sqrt{n}\ell\left(\mathrm{Diam}\left(\mathcal{S}_T\right)\right)+\gamma V^*$, where $V^*\triangleq\mathrm{Vol}\left(\mathcal{S}^*\right)$.
On the other hand, assume one shifts the simplex $\mathcal{S}^*$ to obtain a new simplex $\hat{\mathcal{S}}^*$ which includes the point $\Bar{\boldsymbol{X}}$. Note that the volume of $\hat{\mathcal{S}}^*$ is also $V^*$. However, the following relations hold w.r.t. $\hat{\mathrel{S}}^*$ and for all $i\in\left[n\right]$:
\begin{equation}
d_{\hat{\mathrel{S}}^*}\left(\boldsymbol{X}_i\right)\leq
\min_{\boldsymbol{u}\in\hat{\mathrel{S}}^*}\left\Vert
\boldsymbol{X}_i-\boldsymbol{u}
\right\Vert_2
\leq
\left\Vert
\boldsymbol{X}_i-\Bar{\boldsymbol{X}}
\right\Vert_2
\leq
\mathrm{Diam}\left(\mathcal{S}_T\right),
\end{equation}
which also gives the following upper-bound for the continuously-relaxed risk in \eqref{eq:proposed}: $\hat{R}\left(\hat{\mathcal{S}}^*;\boldsymbol{D}\right)\leq\sqrt{n}\ell\left(\mathrm{Diam}\left(\mathcal{S}_T\right)\right)+\gamma V^*$. This yields that $\hat{\mathcal{S}}^*$ results in a smaller objective value for \eqref{eq:proposed} which contradicts the assumption that $\mathcal{S}^*$ is a global minimizer. Therefore, we must have $\boldsymbol{w}^T_k\Bar{\boldsymbol{X}}+b_k
\leq 2\mathrm{Diam}\left(\mathcal{S}_T\right),~\forall k=0,\ldots,K$, which results in the following inequality for each $k=0,\ldots,K$ and all $\boldsymbol{Z}\in\mathcal{S}_T$:
\begin{equation}
\boldsymbol{w}^T_k\boldsymbol{Z}+b_k=
\boldsymbol{w}^T_k\left(\boldsymbol{Z}-\Bar{\boldsymbol{X}}\right)+
\boldsymbol{w}^T_k\Bar{\boldsymbol{X}}+b_k\leq 3\mathrm{Diam}\left(\mathcal{S}_T\right).
\end{equation}
This also implies that 
$d_{\mathcal{S}^*}\left(\boldsymbol{Z}\right)\leq3\mathrm{Diam}\left(\mathcal{S}_T\right)$
and completes the proof.
\end{proof}

\begin{proof}[Proof of Lemma \ref{lemma2}]
Rademacher complexity of hypothesis set $\Phi$ measures its capability to (in some sense) correlate with a set of $n$ randomly labeled and generally positioned points in the input space. Mathematically speaking, it has the following definition:
\begin{equation}
\mathfrak{R}_n\left\{\Phi\right\}
\triangleq
\mathbb{E}_{\boldsymbol{X}_{1:n}\sim\mathbb{P}_{\mathcal{S}_T}}\left\{
\mathbb{E}_{\boldsymbol{\sigma}}\left(
\sup_{\mathcal{S}\in\mathcal{Q}^c\left(\epsilon\right)}
\frac{1}{n}\sum_{i=1}^{n}\sigma_i \bar{\phi}\left(\boldsymbol{X}_i,\mathcal{S}\right)
\right)
\right\},
\end{equation}
where $\boldsymbol{\sigma}\in\left\{-1,+1\right\}^n$ is a vector of {\it {i.i.d.}} Rademacher random variables, and $\bar{\phi}\left(\cdot,\cdot\right)$ has been already defined in \eqref{eq:gbardef}. Also, remember the set of binary classifiers $\mathscr{C}$ which was defined in the proof of Theorem \ref{corl:ML} in Appendix \ref{appendix:Theorems}; Let us rewrite the definition of the classifier set $\mathscr{C}$, this time using the planar distance $d_{\mathcal{S}}\left(\cdot\right)$:
\begin{equation}
\mathscr{C}\triangleq\left\{
f\left(d_{\mathcal{S}}\left(\cdot\right)\right):\mathbb{R}^K\rightarrow\left\{0,1\right\}
\bigg\vert~
\mathcal{S}\in\mathbb{S}_K
\right\},
\end{equation}
where function $f:\mathbb{R}\rightarrow\left\{0,1\right\}$ is defined as below:
\begin{equation}
f\left(x\right)\triangleq\left\{\begin{array}{lc}
1 & x > 0
\\ 
0 & x \leq 0
\end{array}
\right..
\end{equation}
The binary classifier set $\mathscr{C}$ is different from $\Phi$ in two important aspects. First, functions in $\mathscr{C}$ assign binary labels $\left\{0,1\right\}$ instead of real values in $\left[0,1\right]$. Also, $\mathcal{S}$ is chosen from $\mathbb{S}_K$ instead of $\mathcal{Q}^c\left(\epsilon\right)$ in the case of $\Phi$. This way, we can proceed to the proof by first showing that
\begin{equation}
\mathfrak{R}_n\left\{\Phi\right\}\leq
\mathfrak{R}_n\left\{\mathscr{C}\right\}.
\label{eq:radfoell}
\end{equation}
The reason behind this approach is the fact that r.h.s. of \eqref{eq:radfoell} is much easier to handle in terms of statistical complexity.
In order to show \eqref{eq:radfoell}, we prove a stronger claim. More precisely, we prove that for each $\boldsymbol{\sigma}\in\left\{-1,+1\right\}^n$, the following inequality holds for any dataset $\left\{\boldsymbol{X}_1,\ldots,\boldsymbol{X}_n\right\}$ in $\mathcal{S}_T$:
\begin{equation}
\label{eq:softsimplexclaim}
J\triangleq
\sup_{\mathcal{S}\in\mathcal{Q}^c\left(\epsilon\right)}
\frac{1}{n}\sum_{i=1}^{n}
\sigma_i\bar{\phi}(\boldsymbol{X}_i,\mathcal{S})
\leq
\sup_{\mathcal{S}\in\mathbb{S}_K}
\frac{1}{n}\sum_{i=1}^{n}
\sigma_if(\bar{\phi}(\boldsymbol{X}_i,\mathcal{S}))\triangleq J_{\left(\mathrm{ext}\right)}.
\end{equation}
It should be reminded that $\mathcal{Q}^c\left(\epsilon\right)\subseteq\mathbb{S}_K$, and also $f\left(d_{\mathcal{S}}\left(\boldsymbol{X}_i\right)\right)=f\left(\Bar{\phi}\left(\boldsymbol{X}_i,\boldsymbol{S}\right)\right)$.
Let $\mathcal{S}_{\max}\in\mathbb{S}_K$ represent the maximizer of the l.h.s. of \eqref{eq:softsimplexclaim}. Without loss of generality, assume that:
\begin{align} 
0 
&= \bar{\phi}(\boldsymbol{X}_1, \mathcal{S}_{\max})
= \cdots 
= \bar{\phi}(\boldsymbol{X}_{n_0}, \mathcal{S}_{\max})
\nonumber\\
&< \bar{\phi}(\boldsymbol{X}_{n_0+1}, \mathcal{S}_{\max})
\leq \cdots
\leq \bar{\phi}(\boldsymbol{X}_{n}, \mathcal{S}_{\max}),
\end{align}
where $n_0$ shows the number of data in the dataset such that $\bar{\phi}\left(\boldsymbol{X}_i,\mathcal{S}_{\max}\right)=0$. For the sake of simplicity, assume that for every $i < j$, if $\bar{\phi}\left(\boldsymbol{X}_{i}, \mathcal{S}_{\max}\right) = \bar{\phi}\left(\boldsymbol{X}_{j}, \mathcal{S}_{\max}\right)$, then $\sigma_i \geq \sigma_j$. In this regard, to prove \eqref{eq:softsimplexclaim} one only needs to run the following procedure:
\paragraph{Marking algorithm}
Let the integer numbers between $n_0$ and $n$ to be either {\it {marked}} or {\it {unmarked}}. Also, assume all numbers $\left\{n_0,\ldots,n\right\}$ are unmarked in the beginning, and then increase the index $i$ from $n_0$ to $n$. For each $i$, if $\sigma_i=-1$ define $j$ as the largest unmarked index such that $n_0 \leq j < i$, and $\sigma_j=+1$. If there are no index $j$ which satisfies these conditions, then just pass to the next $i$. Otherwise, match $\sigma_i$ with $\sigma_j$ and set both $i$ and $j$ as marked.

Note that according to the above algorithm, each matched pair contribute a negative value to $J$. Let $n^*$ be the largest index satisfying the conditions: $n^{*} > n_0$ and $\sigma_{n^*}=-1$. If there is no index with these conditions, set $n^{*} = n_0-1$. Then, the following facts hold for $n^*$:
\begin{itemize}
\item For all matched pairs $\left(\sigma_i,\sigma_j\right)$, we have either $\min\left(i,j\right)<n^*$ or $n^*<\max\left(i,j\right)$.
\item For all unmarked $i \leq n^*$, we have $\sigma_i=-1$.
\item For all unmarked $i > n^*$, we have $\sigma_i=+1$.
\end{itemize}
Next, let us define the set
$\mathcal{S}^*_{\max}
\triangleq
\left\{
\boldsymbol{x}\in\mathbb{R}^K\big\vert~
d_{\mathcal{S}_{\max}}\left(\boldsymbol{x}\right)\leq
~\varepsilon\triangleq
d_{\mathcal{S}_{\max}}\left(\boldsymbol{X}_{n^*}\right)
\right\}$. Surprisingly, and based on the definition of the planar distance in Section \ref{sec:notation}, $\mathcal{S}^*_{\max}$ is also a simplex in $\mathbb{S}_K$. Similar to the notation introduced in Section \ref{sec:notation}, let $\left(\boldsymbol{w}_k,b_k\right)$ represent the parameters associated to the hyper-plane which encompasses the $k$th facet of $\mathcal{S}_{\max}$. This way, $\boldsymbol{w}_k\in\mathbb{R}^K$ (with $\left\Vert\boldsymbol{w}_k\right\Vert_2=1$) denotes the normal outward vector associated to the $k$th facet, while $b_k\in\mathbb{R}$ represents the corresponding bias value. Then, it can be easily seen that $\mathcal{S}^*_{\max}$ is also a simplex whose hyper-plane parameters are $\left(\boldsymbol{w}_k,b_k-\varepsilon\right)$, i.e. for each $k$, the $k$th facet is $\varepsilon$-shifted toward its outward normal vector $\boldsymbol{w}_k$. Therefore, and based on the definition of $J_{\left(\mathrm{ext}\right)}$, we can write:
\begin{equation}
J_{\left(\mathrm{ext}\right)}
\ge
\frac{1}{n}\sum_{i=1}^{n}\sigma_i f\left(\Bar{\phi}\left(
\boldsymbol{X}_i,\mathcal{S}^*_{\max}
\right)\right).
\end{equation}
Thanks to the definition of $\mathcal{S}^*_{\max}$, we can now investigate the contribution of each data sample to both $J$ and at least a lower-bound of $J_{\left(\mathrm{ext}\right)}$.
\begin{itemize}
\item Any unmatched data point $\boldsymbol{X}_i$ with $n_0 \leq i \leq n^*$, contributes negatively to $J$, while its contribution to $J_{\left(\text{ext}\right)}$ is at least zero.
\item Any matched pair contribute to $J$ non-positively, while their contribution to $J_{\left(\text{ext}\right)}$ is at least zero.
\item For any unmatched data point $\boldsymbol{X}_i$ with  $i>n^*$, contribution to $J$ is at most $1$, while contribution to $J_{\left(\text{ext}\right)}$ is exactly $1$.
\end{itemize}
In all the above-mentioned cases, contribution of each $\boldsymbol{X}_i$ to $J$ is less than or equal to that of $J_{\left(\text{ext}\right)}$ which proves \eqref{eq:softsimplexclaim}. A direct consequence of these arguments is the following:
\begin{equation}
\mathfrak{R}_n\left\{\Phi\right\}
\leq 
\mathfrak{R}_n\left\{\mathscr{C}\right\},
\end{equation}
since the expectation of $J_{\left(\text{ext}\right)}$ w.r.t. $\boldsymbol{\sigma}$ and data points $\boldsymbol{X}_i$ is in fact the Rademacher complexity of $\mathscr{C}$. For a function class such as $\mathscr{C}$, we can bound the Rademacher complexity via a number of  previously-established tools and measures, e.g. VC-dimension and Massart's lemma. However, we first need to bound the {\it {growth functions}} of $\mathcal{C}$, which has been already carried away in Lemma \ref{lemma:GrowthFunc}. In this regard, according to Massart's lemma (see, for example, Theorem 3.3 of \cite{mohri2012foundations}), and also using \eqref{eq:softsimplexclaim}, we have
\begin{equation}
\mathfrak{R}_n\left\{\Phi\right\}
\leq\sqrt{\frac{2\log
\Pi_{\mathscr{C}}\left(n\right)
}{n}},
\end{equation}
Substitution from Lemma \ref{lemma:GrowthFunc} into the above inequality completes the proof.
\end{proof}


\begin{proof}[Proof of Lemma \ref{lemma:ExpLossBound}]
Lemma roughly states that whenever the overlap between a $K$-simplex $\mathcal{S}$ and $\mathcal{S}_T$ is strictly smaller than $1-\epsilon$, then with high probability $\mathcal{S}$ would not cover all the training data points when $n$ becomes sufficiently large. Thus, the expected value of cost function increases accordingly. First, let us take a closer look at $\mathbb{E}_{\mathcal{S}_T}\hat{R}_{\mathrm{CRR}}$ and analyze its minimum over $\mathcal{S}\in\mathcal{Q}^c\left(\epsilon\right)$:
\begin{equation}
\min_{\mathcal{S}\in\mathcal{Q}^c\left(\epsilon\right)}\mathbb{E}_{\mathcal{S}_T}
\hat{R}_{\mathrm{CRR}}\left(\mathcal{S}\right)
= 
\min_{\mathcal{S}\in\mathcal{Q}^c\left(\epsilon\right)}\left\{
\sqrt{n}\mathbb{E}_{\boldsymbol{X}\sim\mathbb{P}_{\mathcal{S}_T}}\left[
\ell\left(
d_{\mathcal{S}}\left(\boldsymbol{X}\right)
\right)\right]
+\gamma \mathrm{Vol}\left(\mathcal{S}\right)\right\},
\end{equation}
where throughout the proof we refer to $\mathrm{Vol}\left(\mathcal{S}\right)$ simply as $V$ for the sake of simplicity.

The planar distance $d_{\mathcal{S}}\left(\boldsymbol{X}\right)$ is zero when $\boldsymbol{X}\in\mathcal{S}\cap\mathcal{S}_T$. Thus, any nonzero term in the above expectation corresponds to a point in the differential set $\mathcal{S}_T-\mathcal{S}$, which due to the minimization over $\mathcal{S}\in\mathcal{Q}^c\left(\epsilon\right)$, has a Lebesgue measure of at least $\epsilon V_T$. In this regard, let us define the following sets:
\begin{gather}
\mathcal{A}_{k}=\mathcal{A}_{k}\left(\mathcal{S},\mathcal{S}_T\right)\triangleq\left\{
\boldsymbol{X}\in\mathcal{S}_T-\mathcal{S}~\bigg\vert
d_{\mathcal{S}}\left(\boldsymbol{X}\right) = 
\boldsymbol{w}^T_k\boldsymbol{X}+b_k
\right\},
\nonumber\\
\forall k=0,1,\ldots,K,
\end{gather}
where, similar to Section \ref{sec:notation}, $\left(\boldsymbol{w}_k,b_k\right)$ represent the hyper-plane which encompasses the $k$th facet of $\mathcal{S}$. In other words, $\mathcal{A}_k$ shows the set of points inside $\mathcal{S}_T$ and outside of $\mathcal{S}$, which have a smaller planar distance to the $k$th facet of $\mathcal{S}$ compared to its other facets. For each point that has an equal planar distance from more than one facet, $k$ can be chosen as any of those facets, arbitrarily. In this regard, it is easy to see that 
\begin{equation}
\mathcal{A}_k\cap\mathcal{A}_l=\emptyset~\Leftrightarrow ~ k\neq l\quad,\quad
\bigcup_{k=0}^{K}\mathcal{A}_k=\mathcal{S}_T-\mathcal{S}.
\end{equation}
Also, let
\begin{equation}
\alpha_k\triangleq \frac{\mathrm{Vol}\left(\mathcal{A}_k\right)}
{\sum_{l=0}^{K}\mathrm{Vol}\left(\mathcal{A}_l\right)}~,~k=0,1,\ldots,K,
\end{equation}
with $\sum_{k}\alpha_k=1$. Consider the parameteric sets $\mathcal{B}_{\delta}~,~\epsilon\leq\delta\leq1$, defined as
\begin{equation}
\mathcal{B}_{\delta}\triangleq\left\{
\mathcal{S}\in\mathcal{Q}^c\left(\epsilon\right)\bigg\vert~
\frac{\mathrm{Vol}\left(\mathcal{S}\cap\mathcal{S}_T\right)}{V_T}=1-\delta
\right\}.
\end{equation}
Obviously, we have $\bigcup_{\epsilon\leq\delta\leq1}\mathcal{B}_{\delta}=\mathcal{Q}^c\left(\epsilon\right)$. Then, one can show that the following chain of relations hold for all $\epsilon\leq\delta\leq1$:
\begin{align}
\min_{\mathcal{S}\in\mathcal{B}_{\delta}}\left\{
\mathbb{E}_{\mathcal{S}_T}
\hat{R}_{\mathrm{CRR}}\left(\mathcal{S}\right)
-\gamma V
\right\} &= 
\sqrt{n}\min_{\mathcal{S}\in\mathcal{B}_{\delta}}\mathbb{E}_{\mathcal{S}_T}\left[
\ell\left(
d_{\mathcal{S}}\left(\boldsymbol{X}\right)
\right)
\right]
\nonumber
\\
&\ge
\delta\sqrt{n}
\min_{\mathcal{S}\in\mathcal{B}_{\delta}}~
\sum_{k=0}^{K}\alpha_k
\mathbb{E}_{\mathcal{A}_{k}}\left[
\ell\left(
d_{\mathcal{S}}\left(\boldsymbol{X}\right)
\right)
\right]+\left(1-\delta\right)\sqrt{n}\ell\left(0\right)
\nonumber
\\
&\ge \delta\sqrt{n}
\min_{\boldsymbol{\alpha}\in\boldsymbol{\Phi}}~
\sum_{k=0}^{K}\left(\frac{\alpha_k}{\Delta_k}\int_{0}^{\Delta_k}\ell\left(u\right)\mathrm{d}u\right)
+\left(1-\delta\right)\sqrt{n}\ell\left(0\right)
\nonumber
\\
&= \delta\sqrt{n}\min_{\boldsymbol{\alpha}\in\boldsymbol{\Phi}}~
\sum_{k=0}^{K}\alpha_k L\left(\Delta_k\right)+
\sqrt{n}\ell\left(0\right),
\label{eq:ineqChainLemma4}
\end{align}
where $\boldsymbol{\Phi}$ denotes the standard simplex, or the set of all $\left(K+1\right)$-dimensional discrete probability mass functions. $L\left(x\right)\triangleq\frac{1}{x}\int_{0}^{x}\left(\ell\left(u\right) - \ell\left(0\right)\right)\mathrm{d}u$ for $x\ge 0$, and $\Delta_k\triangleq \mathrm{Vol}\left(\mathcal{A}_k\right)/\left(\max_{k}~\mathrm{Vol}\left(\mathcal{S}_{-k}\right)\right)$. The main intuition behind the above inequalities is the fact that the term
\begin{equation}
\mathbb{E}_{\mathcal{A}_{k}}\left[
\ell\left(
d_{\mathcal{S}}\left(\boldsymbol{X}\right)
\right)
\right]
\end{equation}
is minimized when $\mathcal{A}_k$ is concentrated in the form of a thin cylindrical structure above the largest facet of $\mathcal{S}$. The volume of the largest facet is $\max_{k}~\mathrm{Vol}\left(\mathcal{S}_{-k}\right)$, which makes $\Delta_k$ to be the height of this imaginary cylinder. We have also taken advantage of the fact that $\ell$ is increasing and integrable. 
Based on the definition, we have $\mathrm{Vol}\left(\mathcal{A}_k\right)=\delta\alpha_k V_T$. Also, due to the $\left(\underline{\lambda},\bar{\lambda}\right)$-isoperimetricity of $\mathcal{S}_T$, we have 
\begin{equation}
\max_{k}~\mathrm{Vol}\left(\mathcal{S}_{-k}\right)
\leq
\bar{\lambda}K V_T^{\frac{K-1}{K}}
~~\mathrm{which~implies~}~~
\Delta_k \ge \left(\frac{\delta\alpha_k}{\bar{\lambda}}\right)V^{1/K}_T.
\end{equation}
Also, note that since $\ell\left(\cdot\right)$ is assumed to be increasing, $L\left(\cdot\right)$ is both increasing and convex. Therefore, the minimization over $\boldsymbol{\alpha}\in\boldsymbol{\Phi}$ in \eqref{eq:ineqChainLemma4} results in $\boldsymbol{\alpha}^*=\left(\frac{1}{K+1},\cdots,\frac{1}{K+1}\right)$. Substitution into \eqref{eq:ineqChainLemma4} yields the following lower-bound for the expected risk minimized over $\mathcal{B}_{\delta}$:
\begin{align}
\min_{\mathcal{S}\in\mathcal{B}_{\delta}}
\left\{
\mathbb{E}_{\mathcal{S}_T}
\hat{R}_{\mathrm{CRR}}\left(\mathcal{S}\right)
\right\}
\ge&~
\delta\sqrt{n}L\left(\frac{\delta V^{1/K}_T}{\left(K+1\right)\bar{\lambda}}\right) +
\gamma\min_{\mathcal{S}\in\mathcal{B}_{\delta}}\mathrm{Vol}\left(\mathcal{S}\right) + \sqrt{n}\ell\left(0\right)
\nonumber
\\
\ge&~
\delta\sqrt{n}L\left(\frac{\delta V^{1/K}_T}{\left(K+1\right)\bar{\lambda}}\right) -
\gamma V_T\delta + \hat{R}_{\mathrm{CRR}}\left(\mathcal{S}_T\right),
\label{eq:1stBoundLemma2}
\end{align}
where we have used the following simple facts: $\min_{\mathcal{S}\in\mathcal{B}_{\delta}}\mathrm{Vol}\left(\mathcal{S}\right)=V_T\left(1-\delta\right)$, and also $\hat{R}_{\mathrm{CRR}}\left(\mathcal{S}_T\right)=\sqrt{n}\ell\left(0\right) + \gamma V_T$. The rest of the proof is straightforward. Finding a lower-bound for
$\mathbb{E}_{\mathcal{S}_T}\hat{R}_{\mathrm{CRR}}\left(\mathcal{S}\right)$ when $\mathcal{S}\in\mathcal{Q}^c\left(\epsilon\right)$ can be accomplished through the following set of relations and inequalities:
\begin{align}
\min_{\mathcal{S}\in\mathcal{Q}^c\left(\epsilon\right)}\left\{\mathbb{E}_{\mathcal{S}_T}\hat{R}_{\mathrm{CRR}}\left(\mathcal{S}\right)\right\}
&=
\min_{\epsilon\leq\delta\leq 1}\min_{\mathcal{S}\in\mathcal{B}_{\delta}}~
\mathbb{E}_{\mathcal{S}_T}\hat{R}_{\mathrm{CRR}}\left(\mathcal{S}\right)
\nonumber\\
&\ge
\min_{\epsilon\leq\delta\leq 1}\left(
\delta\sqrt{n}L\left(\frac{\delta V^{1/K}_T}{\left(K+1\right)\bar{\lambda}}\right) - 
\gamma V_T\delta
\right)
+ \hat{R}_{\mathrm{CRR}}\left(\mathcal{S}_T\right).
\end{align}
Again, remember that $L\left(\cdot\right)$ is an increasing function. Therefore, given that the derivative of the r.h.s. of  the above inequality remains positive w.r.t. $\delta$, and for all $\epsilon\leq\delta\leq 1$, then the minimum occurs at $\delta^*=\epsilon$. It can be easily verified that given the condition on $n$ in the lemma, the derivative remains positive and thus the proof is complete.
\end{proof}


\begin{lemma}[Restricted Isoperimetry]
\label{lemma:isoperimBound}
For $K\in\mathbb{N}$,
assume $\mathcal{S}_{\mathrm{reg}}\in\mathbb{S}_K$ to be a perfectly regular simplex with equal side lines. Then for all $K$, $\mathcal{S}_{\mathrm{reg}}$ is $\left(\underline{\lambda},\bar{\lambda}\right)$-isoperimetric where
$\underline{\lambda}$ and $\bar{\lambda}$ can be chosen to be as small as $1$ and $e$, respectively.
\end{lemma}
\begin{proof}
For a perfectly regular simplex $\mathcal{S}_{\mathrm{reg}}\left(\boldsymbol{\Theta}\right)$ with all side lines equal to $d\ge0$, i.e. $\left\Vert\boldsymbol{\theta}_k-\boldsymbol{\theta}_{k'}\right\Vert_2=d,~\forall k\neq k'$,
the Lebesgue measure $\mathrm{Vol}\left(\mathcal{S}_{\mathrm{reg}}\right)$ is already shown to be
\begin{equation}
\mathrm{Vol}\left(\mathcal{S}_{\mathrm{reg}}\right)=
\frac{\sqrt{K+1}}{K!\sqrt{2^K}}d^K.
\end{equation}
Thus, we have
\begin{equation}
\max_{k,k'}~\left\Vert\boldsymbol{\theta}_k-\boldsymbol{\theta}_{k'}\right\Vert_2
~=~d~=~
\left(
\frac{K!\sqrt{2^K}}{\sqrt{K+1}}
\right)^{1/K}
\mathrm{Vol}^{\frac{1}{K}}\left(\mathcal{S}_{\mathrm{reg}}\right).
\end{equation}
On the other hand, for all $K\in\mathbb{N}$ we have the following upper-bound for $K!$:
$
K!\leq Ke\left(\frac{K}{e}\right)^K
$,
which leads us to the following relations:
\begin{align}
\max_{k,k'}~\left\Vert\boldsymbol{\theta}_k-\boldsymbol{\theta}_{k'}\right\Vert_2
\leq
\left[
\left(\frac{Ke}{\sqrt{K+1}}\right)^{1/K}
\frac{\sqrt{2}}{e}
\right]
K
\mathrm{Vol}\left(\mathcal{S}_{\mathrm{reg}}\right)^{{1}/{K}}
\leq
K
\mathrm{Vol}\left(\mathcal{S}_{\mathrm{reg}}\right)^{{1}/{K}},
\end{align}
and means $\underline{\lambda}$ can be chosen to be as small as $1$.

For the other claim, we should note that the maximal facet of $\mathcal{S}_{\mathrm{reg}}$ (noting the fact that for a perfectly regular simplex all facets are equal), can be attained by computing the volume of a perfectly regular $\left(K-1\right)$-simplex with all side lines equal to $d$. In other words:
\begin{align}
\max_{k}~\mathrm{Vol}\left(\mathcal{S}_{\mathrm{reg}_{-k}}\right)=
\frac{\sqrt{K}d^{K-1}}{\left(K-1\right)!\sqrt{2^{K-1}}}
=
\frac{\sqrt{K2^{1-K}}}{\left(K-1\right)!}
\left(
\frac{K!\sqrt{2^K}}{\sqrt{K+1}}
\right)^{\frac{K-1}{K}}
\mathrm{Vol}\left(\mathcal{S}_{\mathrm{reg}}\right)^{\frac{K-1}{K}}.
\end{align}
which, again by using the previously-mentioned upper-bound on $K!$, can be bounded as
\begin{align}
\max_{k}~\mathrm{Vol}\left(\mathcal{S}_{\mathrm{reg}_{-k}}\right)
&=
\left[
\sqrt{\frac{K}{\left(K+1\right)^{\frac{K-1}{K}}}}
\left(K!\right)^{-1/K}
\right]
K
\mathrm{Vol}\left(\mathcal{S}_{\mathrm{reg}}\right)^{\frac{K-1}{K}}
\nonumber\\
&\leq
\left[e
\sqrt{\left({K}/{e^2}\right)^{-1/K}}
\right]
\mathrm{Vol}\left(\mathcal{S}_{\mathrm{reg}}\right)^{\frac{K-1}{K}}
\nonumber\\
&\leq
e\mathrm{Vol}\left(\mathcal{S}_{\mathrm{reg}}\right)^{\frac{K-1}{K}},
\end{align}
Again, it means $\bar{\lambda}$ can be chosen to be as small as $e$. This completes the proof.
\end{proof}


\end{document}